\documentclass[11pt]{article}    
\usepackage{tgtermes}
\usepackage{appendix}
\usepackage{enumerate}
\usepackage{bbm}
\usepackage{fancyhdr}
\usepackage{pdfpages}\graphicspath{{/Users/mnk/Pictures/}}
\fancyheadoffset[LE,RO]{\marginparsep+\marginparwidth}
\usepackage{amsmath}
\usepackage{amssymb}
\usepackage{mathptmx}      
\usepackage{color}
\usepackage{mathdots}
\usepackage{dsfont}
\usepackage{pdfsync}
\usepackage{enumitem}
\usepackage{mathtools}
\usepackage{subfig}
\usepackage{endnotes}

\usepackage[numbers]{natbib}
\newcommand{\enumr}{\begin{enumerate}[label=\roman{*})]}
\newcommand{\enumR}{\begin{enumerate}[label=\Roman{*})]}
\newcommand{\enuma}{\begin{enumerate}[label=\alph{*})]}
\renewcommand{\leq}{\leqslant}
  
\newcommand{\bE}{\mathbb{E}}
\newcommand{\bP}{\mathbb{P}}
\newcommand{\bT}{T^{o}}

\renewcommand{\geq}{\geqslant}

\newcommand{\qed}{\hfill $\square$}

\newcommand{\cc}{\citet}

\newcommand{\argmax}{\mbox{\,\rm arg\,max}}

\newcommand{\mnk}[1]{{\color{blue} \  #1 }}
\renewcommand{\mnk}[1]{#1}

\newcommand{\supp}{\mathbf{Sp}}

\newcommand{\Assumption}{Condition\ }  
\newcommand{\Assumptions}{Conditions\ }
\newcommand{\assumption}{condition\ } 
\newcommand{\assumptions}{conditions\ } 
\newcommand{\assumptionsc}{conditions,\ } %
 \newcounter{remn}     
 \newcommand{\rem}{{\bf Remark  \arabic{remn}. \stepcounter{remn}}} 
  
\newtheorem{lemma}{Lemma}
\newtheorem{theorem}{Theorem}
\newtheorem{proposition}{Proposition}

\newcounter{remno} 
{}
\newenvironment{proof}{\noindent {\bf Proof. }}{\hfill  \\}
\newfont{\mfoo}{cmssdc10 scaled\magstep1}
\newfont{\mfo}{cmtt9 scaled\magstep1}

\title{Sequential Experimentation Under Generalized Ranking}

\author{{\bf 
Wesley Cowan} \\
   Department of Mathematics, Rutgers University\\
110 Frelinghuysen Rd., Piscataway, NJ 08854 
\and {\bf Michael N. Katehakis}\\
Department of Management Science and Information Systems\\
 100 Rockafeller Road, Piscataway, NJ 08854, USA
}
\begin{document}
\maketitle
\begin{abstract}
We consider the \mnk{classical}  problem of a controller \mnk{activating (or sampling)} sequentially from a finite number of $N \geq 2$ populations, specified by unknown distributions. Over some time horizon, at each time $n = 1, 2, \ldots$, the controller wishes to select a population to sample, with the goal of sampling from a population that optimizes  some ``score'' function of its distribution, e.g., maximizing the expected sum of outcomes or minimizing  variability. We  
define a class of \textit{Uniformly Fast (UF)} sampling policies and  show, under mild regularity conditions, that there is  an asymptotic  lower bound for the expected  total number of sub-optimal population activations. Then, we provide  
sufficient conditions under which a UCB  policy is UF and  asymptotically 
optimal,    since it attains this lower bound.  Explicit solutions are provided for a number of examples of interest, including general score functionals on unconstrained Pareto distributions (of potentially infinite mean), and uniform distributions of unknown support. Additional results on bandits of Normal distributions are also provided.
\end {abstract}
 
{\bf Keywords:} {Upper Confidence Bound, Multi-armed Bandits, Sequential Allocation, Sequential Experimentation} 

\section{Introduction and Summary}
Let $\mathcal{F}$ be a known family of probability densities on $\mathbb{R}$, and let $\supp(f)$ denote the support of $f$ in $\mathbb{R}$. We consider the problem of a controller sequentially sampling from a finite number of $N \geq 2$ populations or ``bandits'', where measurements from population $i$ are specified by an i.i.d. sequence of random variables $\{ X^i_k \}_{ k \geq 1}$ with density $f_i \in \mathcal{F}$. We take each $f_i$ as unknown to the controller - though the controller is taken to have complete (or at least sufficient) knowledge of $\mathcal{F}$. 

It is often of interest  to maximize the rewards achieved from bandits activated by the controller. While this is often framed in terms of activating the bandit with the largest expected value, this paper is motivated largely by the case of bandits possessing densities with potentially infinite expected values. In this setting, if a controller is given a choice between bandits of infinite mean, by what metric should she choose? Should some infinities be ``preferred'' to others? What loss is incurred when a controller activates a bandit of finite mean in place of one of infinite mean? Additionally, focusing primarily on the ``reward'' of a bandit through its expected value would seem to exclude any consideration of commensurate risk. These considerations, and a general interest in more broad applications, motivate a ``generalized score functional'' approach as follows:

Let $s : \mathcal{F} \mapsto \mathbb{R}$ be a ``score'' functional that maps a probability density to a real number, for example $s(f) = \int_{\supp(f)} x f(x) dx$. 
 For a given $\{ f_i \}_{i = 1}^N \subset \mathcal{F}$, let $s^* = \max_i s(f_i)$ be the maximal realized score, and let $S^* = \{ i : s^* = s(f_i)\} ,$
$S^o = \{ i : s^* > s(f_i)\}$  denote respectively 
  the set of optimal, suboptimal,  bandits.
  
For any adaptive, non-anticipatory policy $\pi$, let $\pi(t) = i$ indicate that the controller samples bandit $i$ at time $t$. Define $T^i_\pi(n) = \sum_{t = 1}^n \boldmath{1}\{ \pi(t) = i \}$, denoting the number of times bandit $i$ has been sampled during the periods $t = 1, \ldots, n$ under policy $\pi$. We take, as a convenience, $T^i_\pi(0) = 0$.

In this generalized setting, it is not immediately clear what the `loss' incurred by sub-optimal activations should be. If the score functional $s$ is taken to be the median, for instance, or the measure of the support of a bandit density, what is `lost' when a sub-optimal bandit is activated in place of an optimal bandit? In this paper, we take the following view of regret, simply that activations of optimal bandits cannot be regretted. We are interested then in policies that minimize the activations of non-optimal bandits, for any choice of bandit distributions in $\mathcal{F}$. Let $\bT_{\pi}(n) = \sum_{i \in S^o} T^i_\pi(n)$ be the total number of sub-optimal activations under $\pi$ up to time $n$. The number of sub-optimal activations up to time $n$ grows at most linearly with $n$, hence in keeping with \cc{bkmab96}, a policy $\pi$ is said to be \textit{Uniformly Fast (UF)} if for all $\delta > 0$,
\begin{equation}
\mathbb{E}\left[ \bT_\pi(n) \right] = o(n^{\delta}), \mbox{  for any choice of $\{ f_i \}_{i = 1}^N \subset \mathcal{F}$.}
\end{equation}
The structure of the rest of the paper is as follows: In Section \ref{sec:optimality}, Theorem \ref{thm:lower-bound} establishes an asymptotic lower  bound on the 
expected total number of sub-optimal activations under any UF policy, under 
two reasonable \assumptions on the structure of $\mathcal{F}$ and $s$. Also in Section \ref{sec:optimality} we define a class of policies $\boldmath \pi^*$ (we call UCB-$(\mathcal{F}, s, \tilde{d})$) specified via  a suitable positive sequence $\tilde{d}(k)$ and easily computed indices $u_i(n, t) $,
and  and provide 
conditions under which such a policy  $\boldmath \pi^*$ is UF and asymptotically optimal in the sense that its  sub-optimal activations achieve the   lower  bound of Theorem \ref{thm:lower-bound}. In addition, we point out that finite horizon bounds and estimates of the asymptotic remainder term on the sub-optimal activations of $\pi^*$, can be easily obtained using the results of 
 therein.   In 
Section \ref{sec:weakening} we discuss weaker conditions and approaches that can be employed  when some of the  conditions required for Theorems \ref{thm:lower-bound} and \ref{thm:optimality} do not hold.  We then demonstrate  asymptotically optimal  $\pi^*$ for:   i) the case of Pareto bandits with a general score functional model cf. Section \ref{sec:pareto}; ii) the case of Uniform bandits over (semi) arbitrary bounded support cf. Section \ref{sec:coverage}; iii) the case of Uniform bandits with unknown interval support and a general score functional model cf. Section \ref{sec:uni}. Finally, in Section \ref{sec:normal}, we consider three  models of Normal bandits under specific score functionals of interest, specifically maximizing the expected value, minimizing the variance, as well as maximizing `tail probabilities' $\mathbb{P}( X_i > \kappa)$ for a given known threshold value $\kappa$.

For related work in this area we refer the reader to:
\cc{Rb52}, and additionally  \cc{gittins-79},   \cc{lai85} and \cc{weber1992gittins} there is a large literature on versions of this problem, cf. \cc{burnetas2003asymptotic}, \cc{burnetas1997finite} and references therein.  For  recent work in this area we refer to   \cc{audibert2009exploration}, 
\cc{auer2010ucb}, \cc{gittins2011multi}, \cc{bubeck2012best},  
\cc{cappe2013kullback}, 
\cc{kaufmann14}, 
\cc{2014minimax},  
\cc{cowan2015multi}, \cc{dena2013}, 
 \cc{honda2011asymptotically},    \cc{honda2010}, and \cc{bkk2015s}.
and references therein. 
Other related work includes:     
\cc{bkmdp97}, \cc{butenko2003cooperative}, \cc{optimistic-mdp}, \cc{audibert2009exploration}, \cc{littman2012inducing},\cc{feinberg2014convergence}, \cc{burnetas1993sequencing}, \cc{BKlarge1996}, \cc{lagoudakis2003least},
\cc{bartlett2009regal}, \cc{tekin2012approximately}, \cc{jouini2009multi},
 \cc{dayanik2013asymptotically}, \cc{filippi2010optimism}, \cc{osband2014near},  and references therein.

\section{Optimality and the Structure of $\left(\mathcal{F}, s\right)$}\label{sec:optimality}
For $f, g \in \mathcal{F}$, with $\supp(g) \supset \supp(f)$, the Kullback-Leibler divergence is defined as
\begin{equation}
\mathbf{I}(f, g) = \bE_f \left[ \ln \left( \frac{f(X)}{g(X)} \right) \right]= \int_{ \supp(f) } \ln \left( \frac{f(x)}{g(x)} \right) f(x) dx.
\end{equation}
While $\mathbf{I}$ is not a metric on $\mathcal{F}$, it is frequently useful as a measure of similarity between $f$ and $g$, effectively measuring how difficult it is to mistake data generated from $f$ to be data from $g$. It is worth noting that $\mathbf{I}(f,g) \geq 0$, and $\mathbf{I}(f,g) = 0$ implies $f = g$ almost everywhere. If $f$ assigns positive weight outside the support of $g$, $\mathbf{I}(f,g)$ is taken to be infinite. In practice, for many $\mathcal{F}$ it follows that that $\mathbf{I}(f,g) < \infty$ implies $\supp(f) \subset \supp(g)$.

It is convenient to define the following function:
\begin{equation}
\begin{split}
\mathbb{M}_f(\rho) & = \inf_{g \in \mathcal{F}} \left\{ \mathbf{I}(f,g) : s(g) > \rho \right\}.
\end{split}
\end{equation}
Thinking of $\mathbf{I}$ as a distance metric, $\mathbb{M}_f(\rho)$ effectively measures how far $f$ must be perturbed to be better than $\rho$ under $s$, a sort of Hausdorff distance. The function $\mathbf{M}_f(\rho)$ captures much of the relevant structure of $(\mathcal{F}, s)$ necessary for asymptotically optimal sampling of bandits.

\subsection{The Lower Bound}
We begin by assuming that \Assumptions B1 and B2 below hold for $\mathcal{F}$ and $s$.
\begin{itemize}
\item {\bf \Assumption B1:} For all $f \in \mathcal{F}$, $\rho \in s(\mathcal{F})$, there exists $\tilde{f} \in \mathcal{F}$ with $s(\tilde{f}) > \rho$ and $\mathbf{I}(f, \tilde{f}) < \infty$.
\end{itemize}
This \assumption means that given a set of bandit distributions $\{ f_i \} \subset \mathcal{F}$, and finite data from each, it is almost surely impossible to correctly identify which bandit is the optimal bandit, i.e., with finite data, any sub-optimal bandit might (somewhat) plausibly be mistaken as an optimal bandit in the set. It serves as a `uniform confusion principle', ensuring the universality of the results to follow for any choice of bandit densities $\{ f_i \} \subset \mathcal{F}$. Additionally, note the technical importance of \Assumption 1, ensuring that $\mathbb{M}_{f_i}(s^*)$ are well defined.

\begin{itemize}
\item {\bf \Assumption B2:} The functional $s$ is continuous with respect to $f$ under $\mathbf{I}$.
\end{itemize}
While $\mathbf{I}$ is not a distance metric, a notion of continuity is easy to define in terms of the usual $(\epsilon, \delta)$-definition. This \assumption then essentially states that if $f$ and $g$ differ slightly (with respect to $\mathbf{I}$), their scores can only differ slightly as well. This will easily and immediately be satisfied by most $\mathcal{F}$ and $s$ \mnk{we are} considering.

We have the following result:
\begin{theorem}\label{thm:lower-bound}
Under \Assumptions 1 and 2, for any UF policy $\pi$ and any choice of $\{ f_i \} \subset \mathcal{F}$, the following bound holds for any sub-optimal bandit $i \notin S^*(\{ f_i \}):$
\begin{equation}\label{eqn:lower-bound-suboptimal}
\liminf_n \frac{ \mathbb{E}\left[ T^i_\pi(n) \right]  }{ \ln n } \geq \frac{1}{ \mathbb{M}_{f_i}( s^* ) },
\end{equation}
and hence
\begin{equation}\label{eqn:lower-bound}
\liminf_n \frac{ \mathbb{E}\left[ \bT_\pi(n) \right]  }{ \ln n } \geq \sum_{i \in S^o} \frac{1}{ \mathbb{M}_{f_i}( s^* ) }.
\end{equation}
\end{theorem}
\begin{proof} Given the above restriction on $\mathcal{F}$, the proof of Eq. \eqref{eqn:lower-bound-suboptimal} proceeds essentially as given in  \cc{bkmab96}. Somewhat technical, and not the focus of the paper, it is relegated to the Appendix. Note that Eq. \eqref{eqn:lower-bound} follows from Eq. \eqref{eqn:lower-bound-suboptimal}, since $\bT_\pi(n) = \sum_{i \in S^o} T^i_\pi(n)$.
\end{proof}

Note that the above result can be applied to bound other loss functions, in particular any linear combination of the activations of sub-optimal bandits, such as the more traditional ``regret'' functions. 

\subsection{Realizing the Bound}
Given this result, it is of interest to construct policies $\pi$, based on knowledge of $\mathcal{F}$ and $s$, that achieve this lower bound, that is $\lim_n \mathbb{E}[ T^i_\pi(n) ]/\ln n = 1/\mathbb{M}_{f_i}(s^*)$ for sub-optimal $i$. These policies are defined to be \textit{Asymptotically Optimal} or \textit{Efficient}, similar to \cc{bkmab96} and \cc{lai85}.

For a given $f \in \mathcal{F}$, let $\hat{f}_t$ be an estimator of $f$ given $t$ i.i.d. samples from $f$. While $\mathbf{I}$ can frequently serve as a similarity measure in $\mathcal{F}$ - for instance, quantifying how close an estimator $\hat{f}_t$ is to $f$ - it is often convenient to consider alternative similarity measures. Let $\mathbf{\nu}$ be a (context-specific) measure of similarity of $\mathcal{F}$; for instance, if $\mathcal{F}$ is parameterized, $\mathbf{\nu}$ might be the $\ell_2$-norm on the parameter space. We restrict $\left(\mathcal{F}, s, \hat{f}_t\right)$ be assuming the following conditions hold, for any $f \in \mathcal{F}$, and all $\epsilon, \delta > 0$,

\begin{itemize}
\item {\bf \Assumption R1:} $\mathbb{M}_f(\rho)$ is continuous with respect to $\rho$, and with respect to $f$ under $\mathbf{\nu}$.
\item {\bf \Assumption R2:} $\mathbb{P}_f( \mathbf{\nu}(\hat{f}_t, f) > \delta ) \leq o(1/t)$.
\item {\bf \Assumption R3:} For some sequence $d_t = o(t)$ (independent of $\epsilon, \delta, f$), $$\mathbb{P}_f( \delta < \mathbb{M}_{ \hat{f}_t }(  s(f) - \epsilon ) ) \leq e^{- \Omega(t)}  e^{-(t-d_t)\delta},$$
where the dependence on $\epsilon$ and $f$ are suppressed into the $\Omega(t)$ term.
\end{itemize}
\Assumption R1 in some sense characterizes the structure of $\mathcal{F}$ as smooth. To the extent that $\mathbb{M}_f(\rho)$ \textit{can} be thought of as a Hausdorff distance on $\mathcal{F}$, \Assumption R1 restricts the ``shape'' of $\mathcal{F}$ relative to $s$. \Assumption R2 is in some sense merely that the estimators $\hat{f}_t$ are ``honest'' and converge to $f$ sufficiently quickly with $t$. \Assumption R3 often seems to be satisfied by $\hat{f}_t$ converging to $f$ sufficiently quickly, as well as $\hat{f}_t$ being ``useful'', in that $s(\hat{f}_t)$ converges sufficiently quickly to $s(f)$. The form of the above bound, while oddly specific in its dependence on $t$ and $\delta$, can be relaxed somewhat, but such a bound frequently seems to exist in practice, for natural choices of $\hat{f}_t$.

In the sequel, for simplicity we will drop the subscript $f$ from $\mathbb{P}_f$,
 when there is no risk for confusion. 
 
 Let $\tilde{d}(t) > 0$ be a non-decreasing function with $\tilde{d}(t) = o(t)$. Define, for any $t$ such that $t > \tilde{d}(t)$, the following index function:
\begin{equation}\label{eqn:index}
u_i(n, t) = \sup_{ g \in \mathcal{F} } \left\{ s(g) :  \mathbf{I}( \hat{f}^i_t, g ) < \frac{\ln n}{t - \tilde{d}(t)}  \right\}.
\end{equation}

For a given $\tilde{d}$, let $n_0 \geq \min\{ n : n > \tilde{d}(n) \}$. We propose the following generic policy:

\textbf{Policy $\boldmath \pi^*$ (UCB-$(\mathcal{F}, s, \tilde{d})$):}
\begin{itemize}
\item[i)] For $n = 1, 2, \ldots, n_0 \times N$, sample each bandit $n_0$ times, and
\item[ii)] for $n \geq n_0 \times N$, sample from bandit $\pi^*(n+1) = \argmax_i u_i\left( n, T^i_{\pi^*}(n) \right)$, breaking ties uniformly at random
\end{itemize}

The following Lemma characterizes the sub-optimal activations of policy $\pi$.

\begin{lemma}\label{lem:bound}
Let $\{ f_i \} \subset \mathcal{F}$ be any choice of bandit densities. Under the above policy, for any sub-optimal $i$ and any optimal $i^*$, the following result holds for any $\epsilon > 0$ such that $s^* - \epsilon > s(f_i)$, and $\delta > 0$ such that $\inf_{g \in \mathcal{F}} \{ \mathbb{M}_g(s^* - \epsilon) : \mathbf{\nu}(g, f_i) \leq \delta \} > 0$:
\begin{equation}
\begin{split}
\mathbb{E}\left[ T^i_{\pi^*}(n) \right] \leq\ & \frac{ \ln n }{ \inf_{g \in \mathcal{F}} \{ \mathbb{M}_g(s^* - \epsilon) : \mathbf{\nu}(g, f_i) \leq \delta \} } + o( \ln n ) \\
& + \sum_{t = n_0 N}^n \mathbb{P}\left( \mathbf{\nu}(\hat{f}^i_t, f_i) > \delta \right) \\
& + \sum_{t = n_0 N}^n \sum_{k = n_0}^t \mathbb{P}\left( u_{i^*}(t, k) \leq s^* - \epsilon \right).
\end{split}
\end{equation}
\end{lemma}
\begin{proof}
The proof is given in the Appendix.
\end{proof}

This leads to the following theorem:
\begin{theorem}\label{thm:optimality}
Let $(\mathcal{F}, s, \hat{f}_t)$ satisfy \Assumptions B1, B2 \& R1 - R3.
Let $d$ be as in \Assumption R3. If $\tilde{d}(t) - d_t \geq \Delta > 0$ for some $\Delta$, for all $t$, then $\pi^*$ is asymptotically optimal. That is, the following holds: For any $\{ f_i \} \subset \mathcal{F}$, for any sub-optimal $i$,
\begin{equation}
\lim_n \frac{ \mathbb{E}\left[ T^i_{\pi^*}(n) \right] }{ \ln n } =  \frac{ 1 }{ \mathbb{M}_{f_i}( s^* ) } .
\end{equation}
\end{theorem}
\begin{proof}
For sub-optimal $i$, there trivially exist feasible $\epsilon$ as in Lemma \ref{lem:bound}. By the continuity of $s$ with respect to $\mathbf{I}$, $\mathbb{M}_f(\rho) > 0$ for all $\rho > s(f)$. It follows from this, and the continuity of $\mathbb{M}_f(\rho)$ with respect to $f$ under $\mathbf{\nu}$ that all sufficiently small $\delta > 0$ are feasible. Let $\epsilon, \delta$ be feasible as in Lemma \ref{lem:bound}.
Note, by \Assumption 4,
\begin{equation}
\sum_{t = 1}^n \mathbb{P}( \mathbf{\nu}( \hat{f}^i_t, f ) > \delta ) \leq \sum_{t = 1}^n o(1/t) \leq o(\ln n).
\end{equation}
Similarly, by \Assumption 5, for $k \geq n_0$,
\begin{equation}
\begin{split}
\mathbb{P}\left( u_{i^*}(t, k) \leq s^* - \epsilon \right) & = \mathbb{P}\left(\sup_{ g \in \mathcal{F} } \left\{ s(g) :  \mathbf{I}( \hat{f}^{i^*}_k, g ) < \frac{\ln t}{k - \tilde{d}(k)}  \right\} \leq s^* - \epsilon \right) \\
& \leq \mathbb{P}\left( \inf_{g \in \mathcal{F} } \left\{ \mathbf{I}(  \hat{f}^{i^*}_k, g ) : s(g) > s^* - \epsilon \right\} > \frac{\ln t}{k - \tilde{d}(k)} \right) \\
& \leq e^{-\Omega(k)}e^{-(k-d(k))\frac{\ln t}{k - \tilde{d}(k)}} \\
& = \frac{1}{t} t^{-\frac{\tilde{d}(k) - d_k}{k - \tilde{d}(k)}} e^{-\Omega(k)}.
\end{split}
\end{equation}
Hence,
\begin{equation}
\sum_{k = n_0}^t \mathbb{P}\left( u_{i^*}(t, k) \leq s^* - \epsilon \right) \leq \sum_{k = n_0}^t \frac{1}{t} t^{-\frac{\tilde{d}(k) - d_k}{k - \tilde{d}(k)}} e^{-\Omega(k)} \leq \frac{1}{t}  \sum_{k = 1}^\infty t^{-\frac{\Delta}{k-\tilde{d}(k)}} e^{-\Omega(k)} \leq \frac{1}{t} O(1/\ln t).
\end{equation}
The last step is proven as Proposition \ref{prop:sum-bound} in the Appendix. From Lemma \ref{lem:bound},
\begin{equation}
\begin{split}
\mathbb{E}\left[ T^i_{\pi^*}(n) \right] & \leq \frac{ \ln n }{ \inf_{g \in \mathcal{F}} \{ \mathbb{M}_g(s^* - \epsilon) : \mathbf{\nu}(g, f_i) \leq \delta \} }  + \sum_{t = 1}^n \frac{1}{t}O(1/\ln t) + o(\ln n) \\
& = \frac{ \ln n }{ \inf_{g \in \mathcal{F}} \{ \mathbb{M}_g(s^* - \epsilon) : \mathbf{\nu}(g, f_i) \leq \delta \} }  + O(\ln \ln n) + o( \ln n ).
\end{split}
\end{equation}
Hence it follows,
\begin{equation}
\limsup_n \frac{ \mathbb{E}\left[ T^i_{\pi^*}(n) \right] }{ \ln n } \leq \frac{ 1 }{ \inf_{g \in \mathcal{F}} \{ \mathbb{M}_g(s^* - \epsilon) : \mathbf{\nu}(g, f_i) \leq \delta \} }.
\end{equation}
By the continuity of $\mathbb{M}$, minimizing the above bound first with respect to $\delta$, then $\epsilon$, yields
\begin{equation}
\limsup_n \frac{ \mathbb{E}\left[ T^i_{\pi^*}(n) \right] }{ \ln n } \leq \frac{ 1 }{ \mathbb{M}_{f_i}(s^*) }.
\end{equation}
By \Assumptions B1 and B2, the proof is completed via the lower bound from Theorem \ref{thm:lower-bound}.
\qed
\end{proof}

For a specific $\mathcal{F}$ and score functional $s$, verifying $\pi^*$ as optimal is reduced to verifying the B-\Assumptions and R-\Assumptions for appropriate choice of estimator $\hat{f}_t$. \Assumptions B1, B2, and R1 are generally easy to verify. In particular, \Assumption R1 seems to follow generally in the case of parameterized $\mathcal{F}$, when $\mathbf{\nu}(f,g)$ depends smoothly on the parameters of $f$ and $g$. \Assumption R2 generally seems to follow for natural estimators. The difficulty often lies in verifying \Assumption R3. 

The   focus of this paper is in demonstrating asymptotic optimality in the spirit of Theorem \ref{thm:optimality}. However, we note that Theorem \ref{thm:optimality} is essentially just an asymptotic upper bound on the results of Lemma \ref{lem:bound}. For specific models, the bounds of Lemma \ref{lem:bound} can be computed more precisely, yielding finite horizon bounds and estimates of the asymptotic remainder term on the sub-optimal activations of $\pi^*$.

\subsection{Weakened \Assumptions and Heterogeneous Bandits}\label{sec:weakening}

\Assumptions B1, B2, \& R1 - R3 above were constructed in such a way as to make the results that followed as universal as possible, relative to the choices of bandit distributions. This has the advantage that the controller may be guaranteed the above results, independent of the specific choice of bandit distributions she is faced with.

However, in some situations, the \assumptions as above may be restrictive. For example, \Assumption B1 precludes any choice of $\mathcal{F}$ and $s$ where the score functional has an attainable maximum over $\mathcal{F}$. This may occur for instance, taking $s(f)$ as the probability that a random variable with density $f$ is greater than or equal to $\kappa$, $s_\kappa(f) = \int_\kappa^{\infty} f(x)dx$, if $\mathcal{F}$ contains densities supported strictly in the interval $[\kappa,\infty)$. In this case, \Assumption B1 would not hold, and the results of Theorems \ref{thm:lower-bound} and \ref{thm:optimality} would not hold.

In such a case, a controller might consider one of two options: In the first, the controller might consider the problem defined over a smaller family of distributions $\mathcal{F}^\prime \subset \mathcal{F}$ where \Assumption B1 could be shown to hold - for instance, $\mathcal{F}^\prime$ might exclude elements of $\mathcal{F}$ that achieve the maximum of $s$. This might be justified in that, given finite samples, the controller might not be able to distinguish a given density in $\mathcal{F}$ from some density in $\mathcal{F}^\prime$.

An alternative though is to consider a less restrictive set of \assumptionsc with less universal results. For instance, the lower bound of Eq. (\ref{eqn:lower-bound-suboptimal}) can be shown to hold for any Uniformly Fast policy, for any set of bandit distributions $\{ f_i \} \subset \mathcal{F}$ that satisfy:

\begin{itemize}
\item {\bf \Assumption $\tilde{\text{B}}$1:} For any sub-optimal $f_j \in \{ f_i \}$, i.e., $s(f_j) \neq s^*(\{ f_i \})$, there exists some $\tilde{f}_j \in \mathcal{F}$ such that $s(\tilde{f}_j) > s^*( \{ f_i \} )$, and $\mathbf{I}(f_j, \tilde{f}_j) < \infty$.
\end{itemize}

This may not hold for all choices of bandit distributions, in a given context, but it may hold for most and in that sense the lower bound of Theorem \ref{thm:lower-bound} might be ``almost universal'' for that choice of $\mathcal{F}$ and $s$. Additionally, in proving Theorem \ref{thm:lower-bound}, \Assumption B2 may be weakened in the following way:
\begin{itemize}
\item {\bf \Assumption $\tilde{\text{B}}$2:} For $f, g \in \mathcal{F}$, if $\mathbf{I}(f, g) = 0$, then $s(f) = s(g)$.
\end{itemize}

However, the continuity of $s$ relative to $\mathbf{I}$ seems necessary for demonstrating the optimality of $\pi^*$, hence \Assumption $\tilde{\text{B}}$2 will not be considered.

\Assumptions R1 and R2 seem fairly natural by themselves and frequently satisfied. The main hurdle in proving the optimality of policy $\pi^*$ as above is \Assumption R3. This may be weakened slightly, in the following way:

\begin{itemize}
\item {\bf \Assumption $\tilde{\text{R}}$3:}  For each $i$, $\sum_{k = 1}^t \mathbb{P}( u_i(t,k) < s(f_i) - \epsilon ) = o(1/t)$.
\end{itemize}
While the order imposed by \Assumption R3 is much stronger than that imposed by $\tilde{\text{R}}3$ above, \Assumption R3 seems to be frequently satisfied, as evidenced by the examples given in the remainder of the paper. Further, \Assumption $\tilde{\text{R}}$3 can be derived from \Assumption R3.

Another way the previous results can be extended is through a heterogeneous bandit model, i.e., the density of bandit $i$ is chosen from some family of densities $\mathcal{F}_i$, $\mathcal{F}_i$ unrelated to $\mathcal{F}_j$ for $i \neq j$. We additionally may equip each individual bandit space with its own score functional $s_i$. In such a model, while specific bandit densities may be unknown, a controller may model information known about individual bandits, e.g., known or assumed parameters. $i$-Specific analogs of \Assumptions B2, R1, R2, R3 may be constructed, for instance with an $i$-specific function $\mathbb{M}^i_f(\rho)$. It is useful to generalize \Assumption B1 in the following way:

\begin{itemize}
\item {\bf \Assumption B$\text{1}^\prime$:} For any choice of bandit distributions $( f_i )_{i = 1}^N \in \bigotimes_{i = 1}^N \mathcal{F}_i$, for each sub-optimal $f_j$, i.e. $s_j(f_j) \neq s^*\left( ( f_i )_{i = 1}^N \right)$, there exists some $\tilde{f}_j \in \mathcal{F}_j$ such that $s_j(\tilde{f}_j) > s^*\left( ( f_i )_{i = 1}^N \right)$ and $\mathbf{I}(f_j, \tilde{f}_j) < \infty$.
\end{itemize}

The results of Theorems \ref{thm:lower-bound} and \ref{thm:optimality} generalize accordingly.

\section{The Pareto Model and Separable Score Functions}\label{sec:pareto}
In this section, we consider a model that demonstrates the utility of this generalized score functional approach. We take $\mathcal{F} = \mathcal{F}_\ell$, for $\ell \geq 0$, as the family of Pareto distributions defined by:
\begin{equation}
\mathcal{F}_\ell = \left\{ f_{\alpha, \beta}(x) = \frac{ \alpha \beta^\alpha}{x^{1 + \alpha}} : \alpha > \ell, \beta > 0 \right\}.
\end{equation}
Taking $X$ as distributed according to $f_{\alpha, \beta} \in \mathcal{F}_\ell$, e.g., $X \sim \text{Pareto}(\alpha,\beta)$, $X$ is distributed over $[\beta, \infty)$, with $\bE[X] = \alpha \beta / (\alpha - 1)$ if $\alpha > 1$, and $\bE[ X ]$ as infinite or undefined if $\alpha \leq 1$. We are particularly interested in $\mathcal{F}_0$, the family of unrestricted Pareto distributions, and $\mathcal{F}_1$, the family of Pareto distributions with finite means.

Taking the general goal of obtaining large rewards from the bandits activated, there are two effects of interests: rewards from a given bandit will be biased towards larger values for increasing $\beta$ and decreasing $\alpha$. Hence, any score function $s(\alpha, \beta) = s(f_{\alpha, \beta})$ of interest should be an increasing (or at least non-decreasing) function of $\beta$, and a decreasing (or at least non-increasing function of $\alpha$. In particular, we restrict our attention to score functions that are ``separable'' in the sense that
\begin{equation}\label{Eq:sab}
s(\alpha, \beta) = a(\alpha)b(\beta),
\end{equation}
where we take $a$ to be a positive, continuous, decreasing, invertible function of $\alpha$ for $\alpha > \ell$, and $b$ to be a positive, continuous, non-decreasing function of $\beta$.

\rem\label{3scores}\ 
This general Pareto model of Eq. (\ref{Eq:sab}) includes several natural score functions of interest, in particular:
\begin{itemize}
\item[i)] In the case of the restricted Pareto distributions with finite mean, we may take $s$ as the expected value, and  $s(\alpha, \beta) = \alpha \beta/(\alpha - 1)$, with $a(\alpha) = \alpha/(\alpha - 1)$ and $b(\beta) = \beta$. 
\item[ii)] In the case of unrestricted Pareto distributions, various asymptotic considerations give rise to considering the score function $s(\alpha, \beta) = 1/\alpha$, i.e., the controller attempts to find the bandit with minimal $\alpha$. In this case, $a(\alpha) = 1/\alpha$ and $b(\beta) = 1$. This arises for instance in comparing the asymptotic tail distributions of bandits, $\bP(X \geq k)$ as $k \to \infty$, or the conditional restricted expected values, $\bE[ X | X \leq k]$ as $k \to \infty$. 
\item[iii)] A third score function to consider is the median, defined over unrestricted Pareto distributions, with $s(\alpha, \beta) = \beta 2^{1/\alpha}$, taking $a(\alpha) = 2^{1/\alpha},\ b(\beta) = \beta$.
\end{itemize}

Given the above special cases, it is convenient to take the assumption when operating over $\mathcal{F}_\ell$ that $a(\alpha) \to \infty$ as $\alpha \to \ell$. This has the advantage additionally of guaranteeing that \Assumption B1 is satisfied for this choice of score function $s$ over $\mathcal{F}_\ell$.

For $f = f_{\alpha, \beta} \in \mathcal{F}_\ell$, and $t$ many i.i.d. samples under $f$, take the estimator $\hat{f}_t = f_{\hat{\alpha}_t, \hat{\beta}_t}$ where
\begin{equation}\label{eqn:estimators}
\begin{split}
\hat{\beta}_t & = \min_{n = 1, \ldots, t} X_n, \\
\hat{\alpha}_t & = \frac{t-1}{ \sum_{n = 1}^t \ln\left( \frac{X_n}{\hat{\beta}_t} \right)}.
\end{split}
\end{equation}

At various points in what follows, it is convenient to define the following functions, $L^{+}(\delta),\ L^{-}(\delta)$, as respectively the smallest and largest positive solutions to $L - \ln L - 1 = \delta$ for $\delta \geq 0$. In particular, $L^{-}(\delta)$ may be expressed in terms of the Lambert-$W$ function, $L^{-}(\delta) = - W(e^{-1-\delta})$, taking $W(x)$ be the principal solution to $We^W = x$ for $x \in [-1/e, \infty)$. An important property will be that $L^{\pm}(\delta)$ is continuous as a function of $\delta$, and $L^{\pm}(\delta) \to 1$ as $\delta \to 0$.

Given the above, we may define the following policy as a specific instance of policy $\pi^*$ under this model:

\textbf{Policy $\boldmath \pi^*_{\text{P,s}}$ (UCB-PARETO)}
\begin{itemize}
\item[i)] For $n = 1, 2, \ldots 3N$, sample each bandit $3$ times, and
\item[ii)] for $n \geq 3N$, sample from bandit $\pi^*_{\text{P,s}}(n+1) = \argmax_i u_i\left( n, T^i_{\pi^*_{\text{P,s}}}(n) \right)$ breaking ties uniformly at random, where
\begin{equation}
u_i(n, t) = \begin{cases} \infty &\mbox{if } \hat{\alpha}^i_t L^{-}\left( \frac{\ln n}{t - 2} \right) \leq \ell \\ 
 b\left( \hat{\beta}^i_{ t } \right) a\left( \hat{\alpha}^i_t L^{-}\left( \frac{\ln n}{t - 2} \right)\right) & else. \end{cases}
 \end{equation}
\end{itemize}

\begin{theorem}
Policy $\pi^*_{\text{P,s}}$ as defined above is asymptotically optimal.  In particular, for any choice of $\{ f_i = f_{\alpha_i, \beta_i} \} \subset \mathcal{F}_\ell$, with $s^* = \max_i s(\alpha_i, \beta_i) = \max_i a(\alpha_i) b(\beta_i)$, for each sub-optimal bandit $i$ the following holds:
\begin{equation}\label{eqn:pareto-limit}
\lim_n \frac{ \mathbb{E}\left[ T^i_{\pi^*_{\text{P,s}}}(n) \right] }{ \ln n } = \frac{ 1 }{ \frac{1}{\alpha_i} a^{-1}\left( \frac{s^*}{ b(\beta_i) } \right) - \ln\left( \frac{1}{\alpha_i} a^{-1}\left( \frac{s^*}{ b(\beta_i) } \right) \right) -1 }.
\end{equation}
\end{theorem}
\begin{proof}
It suffices to verify \Assumptions B1, B2, \& R1-R3 for the indicated Pareto model. To begin, it can be shown that
\begin{equation}\label{eqn:a-and-i-pareto}
\begin{split}
\mathbf{I}( f_{\alpha, \beta}, f_{\tilde{\alpha}, \tilde{\beta}} ) & =
\begin{cases}
\frac{\tilde{\alpha}}{\alpha} - \ln\left( \frac{\tilde{\alpha}}{\alpha}\right) -1 + \tilde{\alpha}\ln\left( \frac{\beta}{\tilde{\beta}}\right) &\mbox{if } \tilde{\beta} \leq \beta \\ 
\infty & \mbox{else,}
\end{cases}\\
\mathbb{M}_{f_{\alpha, \beta}}(\rho) & = 
\begin{cases}
\frac{1}{\alpha} a^{-1}\left( \frac{\rho}{ b(\beta) } \right) - \ln\left( \frac{1}{\alpha} a^{-1}\left( \frac{\rho}{ b(\beta) } \right) \right) -1 & \mbox{if } \rho > s(\alpha, \beta) \\
0 & \mbox{ else.}
\end{cases}.
\end{split}
\end{equation}
Given the above, \Assumption B1 is easy to verify given the structure of the score function. Additionally, note that $\mathbf{I}( f_{\alpha, \beta}, f_{\tilde{\alpha}, \tilde{\beta}} ) < \delta$ implies that
\begin{equation}
\begin{split}
\tilde{\beta} & \leq \beta \\
\frac{\tilde{\alpha}}{\alpha} - \ln\left( \frac{\tilde{\alpha}}{\alpha}\right) -1 & \leq \delta \\
\tilde{\alpha}\ln\left( \frac{\beta}{\tilde{\beta}}\right) & \leq \delta.
\end{split}
\end{equation}
The above gives us that $\alpha L^{-}(\delta) \leq \tilde{\alpha} \leq \alpha L^{+}(\delta)$ and $\beta e^{-\alpha \delta L^{+}(\delta)} \leq \tilde{\beta} \leq \beta$. Given that $\delta L^{+}(\delta) \to 0$ as $\delta \to 0$, these bounds and the continuity of $a, b$, give the continuity of $s$ with respect to $\mathbf{I}$, verifying \Assumption B2. 

In verifying \Assumptions R1 - R3, it is convenient to take as similarity measure on $\mathcal{F}_\ell$, $\mathbf{\nu} = \mathbf{I}$. \Assumption R1 is then easily verified, the continuity of $\mathbb{M}_f(\rho)$ with respect to $\rho$ from the above formula, and the continuity with respect to $f$ under $\mathbf{I}$ from the previous bounds.

In verifying \Assumption R2, it is interesting to note that for $\ell > 0$, the estimator $\hat{f}_t = f_{\hat{\alpha}_t, \hat{\beta}_t}$ of $f = f_{\alpha, \beta}$ may not be in $\mathcal{F}_\ell$ even if $f$ is, i.e., even if $\alpha > \ell$, there is no immediate guarantee that $\hat{\alpha}_t$ is. Hence, $\mathbf{I}(\hat{f}_t, f)$ may not be well defined over $\mathcal{F}_\ell$. However, this is not a serious issue as in the case that $\ell > 0$, we may view this as embedded in $\mathcal{F}_0$, which will contain $\hat{f}_t$, and hence allow us to compute $\mathbf{I}(\hat{f}_t, f)$. Hence, for $\delta > 0$, since $\hat{\beta}_t \geq \beta$,
\begin{equation}
\begin{split}
\mathbb{P}\left( \mathbf{I}(\hat{f}_t, f) > \delta \right) & = \mathbb{P}\left( \frac{\alpha}{\hat{\alpha}_t} - \ln\left( \frac{\alpha}{\hat{\alpha}_t}\right) -1 + \alpha\ln\left( \frac{\hat{\beta}_t}{\beta}\right) > \delta \right) \\
& \leq \mathbb{P}\left( \frac{\alpha}{\hat{\alpha}_t} - \ln\left( \frac{\alpha}{\hat{\alpha}_t}\right) -1  > \frac{\delta}{2} \right) + \mathbb{P}\left(\alpha\ln\left( \frac{\hat{\beta}_t}{\beta}\right) > \frac{\delta}{2} \right) \\
& = \mathbb{P}\left( \frac{\alpha}{\hat{\alpha}_t} < L^{-}\left( \frac{\delta}{2} \right) \right) + \mathbb{P}\left( \frac{\alpha}{\hat{\alpha}_t} > L^{+}\left( \frac{\delta}{2} \right) \right)  + \mathbb{P}\left(\frac{\hat{\beta}_t}{\beta} > e^{\frac{\delta}{2\alpha}} \right).
\end{split}
\end{equation}
At this point, we make use of the following result, characterizing the distributions of $\hat{\alpha}_t$ and $\hat{\beta}_t$:
\begin{lemma}\label{lem:estimators}
With $\hat{\alpha}_t, \hat{\beta}_t$ as in Eq. \eqref{eqn:estimators}, $\hat{\alpha}_t$ and $\hat{\beta}_t$ are independent, with
\begin{equation}
\begin{split}
\frac{\alpha}{\hat{\alpha}_t}(t-1) & \sim \text{Gamma}(t-1,1),\\
\frac{\hat{\beta}_t}{\beta} & \sim \text{Pareto}(\alpha t, 1).
\end{split}
\end{equation}
\end{lemma}
The proof is given in the Appendix.

It follows, letting $G_t \sim \text{Gamma}(t,1)$,
\begin{equation}
\mathbb{P}\left( \mathbf{I}(\hat{f}_{t}, f) > \delta \right) \leq \mathbb{P}\left( G_{t-1} < (t-1) L^{-}\left( \delta/2 \right) \right) + \mathbb{P}\left( G_{t-1} > (t-1)L^{+}\left( \delta/2 \right) \right)  + e^{-\frac{\delta}{2}t}.
\end{equation}
Here we apply the following result, bounding the tails of the Gamma distributions:
\begin{lemma}\label{lem:gamma-bound}
Let $G_t \sim \text{Gamma}(t,1)$. For $0 < \gamma^{-} < 1 < \gamma^{+} < \infty$, the following bounds hold:
\begin{equation}
\begin{split}
\mathbb{P}\left( G_{t} < t \gamma^{-}  \right) & \leq \left( \gamma^{-} e^{1 - \gamma^{-}} \right)^t \\
\mathbb{P}\left( G_{t} > t \gamma^{+} \right) & \leq \left( \gamma^{+} e^{1 - \gamma^{+}} \right)^t.
\end{split}
\end{equation}
\end{lemma}
These are standard Chernoff bounds, proven in the Appendix. Applying them to the above, taking $\gamma^{\pm} = L^{\pm}(\delta/2)$, note that $\gamma^{\pm} e^{1 - \gamma^{\pm}} = e^{-\delta/2}$. Hence,
\begin{equation}
\mathbb{P}\left( \mathbf{I}(\hat{f}_{t}, f) > \delta \right) \leq 2 e^{-\frac{\delta}{2}(t-1)}  + e^{-\frac{\delta}{2}t} = (2 e^{\frac{\delta}{2}}  + 1)e^{-\frac{\delta}{2}t} = e^{-O(t)}.
\end{equation}
This verifies \Assumption R2 - to a much faster rate than is in fact required. It remains to verify \Assumption R3. For $\delta > 0$,
\begin{equation}
\begin{split}
& \mathbb{P}( \delta < \mathbb{M}_{ \hat{f}_t }(  \rho ) ) \\
& = \mathbb{P}\left( \delta < \frac{1}{\hat{\alpha}_t} a^{-1}\left( \frac{\rho}{ b(\hat{\beta}_t) } \right) - \ln\left( \frac{1}{\hat{\alpha}_t} a^{-1}\left( \frac{\rho}{ b(\hat{\beta}_t) } \right) \right) -1 \text{ and }  \frac{\rho}{b(\hat{\beta}_t)}  > a(\hat{\alpha}_t) \right) \\
& = \mathbb{P}\left( \frac{\rho}{ b(\hat{\beta}_t) }  > a( \hat{\alpha}_t L^{-}(\delta) ) \text{ and }  \frac{\rho}{b(\hat{\beta}_t)}  > a(\hat{\alpha}_t) \right) \\
&  + \mathbb{P}\left(\frac{\rho}{ b(\hat{\beta}_t) } < a( \hat{\alpha}_t L^{+}(\delta) )  \text{ and }  \frac{\rho}{b(\hat{\beta}_t)}  > a(\hat{\alpha}_t) \right).
\end{split}
\end{equation}
The above bound can be simplified a great deal. In the second term, the conditions in fact contradict, since $a$ is taken to be a decreasing function of $\alpha$, and $L^{+}(\delta) > 1$ for $\delta > 0$, hence the probability is $0$. In the first term, since $0 < L^{-}(\delta) < 1$ for $\delta > 0$, and $a$ is decreasing, the conditions may be combined to yield
\begin{equation}
\mathbb{P}( \delta < \mathbb{M}_{ \hat{f}_t }(  \rho ) ) = \mathbb{P}\left( \frac{\rho}{ b(\hat{\beta}_t) }  > a( \hat{\alpha}_t L^{-}(\delta) ) \right).
\end{equation}

Let $\rho = s(f) - \epsilon = a(\alpha)b(\beta) - \epsilon$. It is convenient to take $\epsilon = a(\alpha)b(\beta)\tilde{\epsilon}$ with $0 < \tilde{\epsilon} < 1$, so $\rho = a(\alpha)b(\beta)(1-\tilde{\epsilon})$. Recall that $b$ is non-decreasing, and $\beta \leq \hat{\beta}_t$. Hence,
\begin{equation}
\begin{split}
\mathbb{P}( \delta < \mathbb{M}_{ \hat{f}_t }(  s(f) - \epsilon  ) ) & = \mathbb{P}\left( \frac{a(\alpha)b(\beta)(1-\tilde{\epsilon})}{ b(\hat{\beta}_t) }  > a( \hat{\alpha}_t L^{-}(\delta) ) \right)\\
& \leq \mathbb{P}\left( a(\alpha)(1-\tilde{\epsilon}) > a( \hat{\alpha}_t L^{-}(\delta) ) \right) \\
& = \mathbb{P}\left( \frac{\alpha}{\hat{\alpha}_t}  <  \frac{ \alpha }{a^{-1}\left( a(\alpha)(1-\tilde{\epsilon}) \right) } L^{-}(\delta) \right)
\end{split}
\end{equation}
Let $\sigma = \alpha / a^{-1}( a(\alpha)(1 - \tilde{\epsilon} ))$, and note that by \assumption on $a$, $0 < \sigma < 1$. Letting $G_t \sim \text{Gamma}(t,1)$, we may apply Lemma \ref{lem:gamma-bound} for
\begin{equation}
\mathbb{P}( \delta < \mathbb{M}_{ \hat{f}_t }(  s(f) - \epsilon  ) )  \leq \mathbb{P}\left( G_{t-1}  < (t-1) \sigma L^{-}(\delta) \right) \leq \left( \sigma L^{-}(\delta) e^{1 - \sigma L^{-}(\delta)} \right)^{t-1}
\end{equation}
Noting that $L^{-}(\delta) - \ln L^{-}(\delta) - 1 = \delta$, we have $L^{-}(\delta) e = e^{L^{-}(\delta) - \delta}$, and
\begin{equation}
\mathbb{P}( \delta < \mathbb{M}_{ \hat{f}_t }(  s(f) - \epsilon  ) )  \leq \left( \sigma e^{L^{-}(\delta)(1-\sigma) - \delta} \right)^{t-1}\leq \left( \sigma e^{1-\sigma}\right)^{t-1}e^{- \delta(t-1)}. 
\end{equation}
The last step follows as $0 < L^{-}(\delta) < 1$ for $\delta > 0$. This verifies \Assumption R3, with $d_t = 1$, producing a bound of the correct order. This in turn verifies the policy as optimal, taking $\tilde{d}(t) = 2$, and Eq. \ref{eqn:pareto-limit} follows from Eq. \ref{eqn:a-and-i-pareto}, the definition of $\mathbb{M}_f(\rho)$ for this model.
\qed
\end{proof}

\section{Maximizing Coverage of (Bounded) Uniform Bandits}\label{sec:coverage}
In this section, we consider a bandit model that demonstrates the necessity of the general form of \Assumption R3. In particular, consider the set of distributions that are uniform over finite disjoint unions of closed sub-intervals of $[0,1]$, i.e.,
\begin{equation}
\mathcal{F} = \left\{ f_S = \mathbbm{1}\left\{ x \in S \right\}/\lvert S \rvert : S = \bigcup_{i = 1}^k \left[ a_i, b_i \right], 0 \leq a_i < b_i \leq 1, k < \infty \right\}.
\end{equation}
For $S$ as above, it is convenient to take $\lvert S \rvert = \sum_{i = 1}^k (b_i - a_i)$ as the measure of $S$. Note that over this class of distributions, we have the following, that
\begin{equation}
\mathbf{I}(f_S, f_T) = 
\begin{cases} \ln( \lvert T \rvert / \lvert S \rvert ) &\mbox{if } S \subset T \\ 
\infty & \mbox{else}\end{cases}.
\end{equation}

We take as the score functional $s(f_S) = \lvert S \rvert$, the area covered by a given distribution in $\mathcal{F}$. In order to satisfy \Assumption B1, however, it is necessary to remove the complete interval $\left[ 0, 1 \right]$ from consideration, so we take $\mathcal{F}^{\prime} = \mathcal{F} \setminus \left\{ \left[ 0, 1 \right] \right\}$.

Under this model, we therefore have (noting that we are only interesting in $\rho \leq 1$),
\begin{equation}
\mathbb{M}_{f_S}(\rho) = 
\begin{cases} \ln( \rho / \lvert S \rvert ) &\mbox{if } \rho >  \lvert S \rvert \\ 
0 & \mbox{else}\end{cases}.
\end{equation}

Given $t$ samples from $f_S$, take for the moment $\hat{S}_t$ to be an estimate of $S$, that may or may not cover all of $S$. The fact that it is impossible to know if a non-trivial estimate for $S$ contains or is contained by $S$ makes using $\mathbf{I}$ as a measure of similarity difficult, as an estimate may be quite close to the truth, and yet have infinite difference under $\mathbf{I}$ - and this may not be uncommon. This prompts an alternative similarity measure, $\mathbf{\nu}(f_S, f_T) = \lvert \lvert S \rvert - \lvert T \rvert \rvert$. Note, this $\mathbf{\nu}$ is in fact a pseudo-metric on $\mathcal{F}^\prime$, but it will prove sufficient for our purposes. For any system of estimators of $S$, and some $\tilde{d}(t)$, we have as our index from Eq. \eqref{eqn:index},
\begin{equation}\label{eqn:index-eqn}
u_i(n, t) = \max \left(  \lvert \hat{S}_t \rvert n^{ \frac{1}{t - \tilde{d}(t)} }, 1 \right).
\end{equation}
At this point, the B-\Assumptions and \Assumption R1 are easily verified. \Assumptions R2 and R3 depend on the specifics of the estimators. We take the following scheme for estimating the support: Let $d_k$ be a positive, integer valued, non-decreasing function that is unbounded and sub-linear $k$. Given $t$ samples from $f_S$, consider a partition of $\left[ 0, 1 \right]$ into a sequence of intervals of width $\epsilon_t = 1/d_t$. The estimator $\hat{S}_t$ is then taken to be the union of partition intervals that contain at least one sample of the $t$ samples.

\Assumption R2 now takes the following form:
\begin{equation}
\mathbb{P}\left( \lvert \lvert \hat{S}_t \rvert - \lvert S \rvert \rvert > \delta \right) = o(1/t).
\end{equation}
Observe the decomposition,
\begin{equation}
\mathbb{P}\left( \lvert \lvert \hat{S}_t \rvert - \lvert S \rvert \rvert > \delta \right) = \mathbb{P}\left( \lvert \hat{S}_t \rvert  > \lvert S \rvert + \delta \right) + \mathbb{P}\left(  \lvert \hat{S}_t \rvert <  \lvert S \rvert - \delta \right).
\end{equation}
We have the following bound, almost surely, on the size of $\hat{S}_t$: Letting $\#S$ denote the number of disjoint intervals in $S$, $\lvert \hat{S}_t \rvert \leq \lvert S \rvert + 2 \epsilon_t \#S$. As this is almost sure, and $\epsilon_t \to 0$ with $t$, the first term in the decomposition above is $0$ for all sufficiently large $t$. To bound the other term, note that without loss of generality, we may take $\delta < \lvert S \rvert$. For notational convenience, let $\alpha = 1 - \delta/\lvert S \rvert$, and note that $0 < \alpha < 1$.

In the event that $\lvert \hat{S}_t \rvert <  \alpha \lvert S \rvert$, there exists a set of $\epsilon_t$-intervals of those that intersect $S$ that both cover a total measure of $\alpha \lvert S \rvert$, and contain all $t$ samples from $f_S$. The number of $\epsilon_t$-intervals intersecting $S$ is at most $\lceil \lvert S \rvert / \epsilon_t \rceil + 2\#S$. The number of $\epsilon_t$-intervals to cover an area of $\alpha \lvert S \rvert$ is $\lceil \alpha \lvert S \rvert / \epsilon_t \rceil$. Noting that the $f_S$ samples are independent, and fall in a given set of $\alpha \lvert S \rvert$-covering $\epsilon_t$-intervals with probability at most $\alpha$, we have
\begin{equation}
\begin{split}
\mathbb{P}\left( \lvert \hat{S}_t \rvert <  \alpha \lvert S \rvert \right) & \leq \binom{\lceil \lvert S \rvert d_t \rceil + 2\#S}{\lceil \alpha \lvert S \rvert  d_t \rceil} \alpha^t \\
& \leq \left( e \frac{\lceil \lvert S \rvert d_t \rceil + 2\#S}{\lceil \alpha \lvert S \rvert  d_t \rceil} \right)^{\lceil \alpha \lvert S \rvert  d_t \rceil} \alpha^t \\
& \leq \left( e \frac{ \lvert S \rvert d_t + 2\#S + 1}{\alpha \lvert S \rvert  d_t } \right)^{d_t} \alpha^t \\
& = \left(1 +  \frac{2\#S + 1}{\lvert S \rvert  d_t } \right)^{d_t} e^{d_t} \alpha^{t- d_t} = e^{O(d_t)} \alpha^{t - d_t}.
\end{split}
\end{equation}
It follows from this and the previous analysis that $\mathbb{P}\left( \lvert \lvert \hat{S}_t \rvert - \lvert S \rvert \rvert > \delta \right) = e^{-\Omega(t)}$ in fact, verifying \Assumption R2.
To verify \Assumption R3, note
\begin{equation}
\mathbb{P}( \delta < \mathbb{M}_{ \hat{f}_t }(  s(f) - \epsilon ) ) \leq \mathbb{P}\left( \delta < \ln\left( \lvert S \rvert - \epsilon ) / \lvert \hat{S}_t \rvert \right) \right) = \mathbb{P}\left( \lvert \hat{S}_t \rvert < \left( \lvert S \rvert - \epsilon \right) e^{-\delta} \right).
\end{equation}
The additional case in $\mathbb{M}_f$ may be dispensed with observing that $\delta > 0$. Taking $\epsilon < \lvert S \rvert$, it is convenient to define $\tilde{\epsilon} = 1 - \epsilon / \lvert S \rvert$. In which case,
\begin{equation}
\mathbb{P}( \delta < \mathbb{M}_{ \hat{f}_t }(  s(f) - \epsilon ) ) \leq \mathbb{P}\left( \lvert \hat{S}_t \rvert < \lvert S \rvert \left( 1 - \tilde{\epsilon} \right) e^{-\delta} \right).
\end{equation}
Applying the previously established bound therefore yields,
\begin{equation}
\mathbb{P}( \delta < \mathbb{M}_{ \hat{f}_t }(  s(f) - \epsilon ) ) \leq e^{O(d_t)} (1 - \tilde{\epsilon})^{ t - d_t } e^{-\delta(t - d_t)},
\end{equation}
verifying \Assumption R3.

We may now present the following result: Let $d_t = o(t)$ be an positive, integer valued, non-decreasing, unbounded sequence, and define $\tilde{d}(t) = d_t + 1$. Let $n_0 = \min \{ n : n > \tilde{d}(n) \}$.

\textbf{Policy $\boldmath \pi^*_{\text{U},\lvert \rvert}$ (UCB-COVERAGE)}
\begin{itemize}
\item[i)] For $n = 1, 2, \ldots n_0 \times N$, sample each bandit $n_0$ times, and
\item[ii)] for $n \geq n_0 \times N$, sample from bandit $\\pi^*_{\text{U},\lvert \rvert}(n+1) = \argmax_i u_i\left( n, T^i_{\pi^*_{\text{U},\lvert \rvert}}(n) \right)$ breaking ties uniformly at random, where
\begin{equation}
u_i(n, t) =  \max \left(  \lvert \hat{S}_t \rvert n^{ \frac{1}{t - \tilde{d}(t)} }, 1 \right)
 \end{equation}
\end{itemize}

\begin{theorem}
Policy $\pi^*_{\text{U},\lvert \rvert}$ as defined above is asymptotically optimal.  In particular, for any choice of $\{ f_i = f_{S_i} \} \subset \mathcal{F}^\prime$, with $s^* = \max_i s(f_{S_i}) = \max_i \lvert S_i \rvert$, for each sub-optimal bandit $i$ the following holds:
\begin{equation}\label{eqn:pareto-limit}
\lim_n \frac{ \mathbb{E}\left[ T^i_{\pi^*_{\text{U},\lvert \rvert}}(n) \right] }{ \ln n } = \frac{ 1 }{  \ln s^* - \ln \lvert S_i \rvert }.
\end{equation}
\end{theorem}
\begin{proof}
The proof is given above, through the verification of The B- and R-\Assumptions.
\end{proof}

\section{The Uniform Model under General Score Functionals}\label{sec:uni}
In this section, the uniform distributions are taken to be over single intervals, with finite but otherwise unconstrained bounds. This additional restriction on the support is necessary to ensure that the score functionals considered will be continuous with respect to $\mathbf{I}$. We take $\mathcal{F}$ as the family of Uniform distributions with interval support:
\begin{equation}
\mathcal{F} = \left\{ f_{a,b}(x) = \mathbbm{1}\{ x \in [a,b] \}/(b-a) : -\infty < a < b < \infty \right\}.
\end{equation}
Taking $X$ as distributed according to $f_{a, b} \in \mathcal{F}$, e.g., $X \sim \text{Unif}[a,b]$, $X$ is distributed over $[a, b]$, with $\bE[X] = (a + b)/2$. As this is a well defined function over all of $\mathcal{F}$, it makes for a reasonable (and traditional) score functional. However, we are aiming for greater generality. Taking the controller's goal to be achieving large rewards from the activated bandits, any score functional $s(f_{a,b}) = s(a,b)$ of interest should be an increasing function of $a$, and an increasing function of $b$. We additionally take $s$ to be continuous in $a$ and $b$. Note, this is satisfied taking $s$ as the expected value, $s_\mu(a,b) = (a + b)/2$.

For $f = f_{a,b} \in \mathcal{F}$, and $t$ many i.i.d. samples under $f$, we take the estimator $\hat{f}_t = f_{\hat{a}_t, \hat{b}_t} \in \mathcal{F}$, where
\begin{equation}
\begin{split}
\hat{a}_t & = \min_{n = 1, \ldots, t} X_n, \\
\hat{b}_t & = \max_{n = 1, \ldots, t} X_n.
\end{split}
\end{equation}

Given the above, we may define the following policy as a specific instance of policy $\pi^*$ under this model:

\textbf{Policy $\boldmath \pi^*_{\text{U,s}}$ (UCB-UNIFORM)}
\begin{itemize}
\item[i)] For $n = 1, 2, \ldots 3N$, sample each bandit $3$ times, and
\item[ii)] for $n \geq 3N$, sample from bandit $\pi^*_{\text{U,s}}(n+1) = \argmax_i u_i\left( n, T^i_{\pi^*_{\text{U,s}}}(n) \right)$ breaking ties uniformly at random, where
\begin{equation}
u_i(n, t) = s( \hat{a}^i_t, \hat{a}^i_t + n^{\frac{1}{t-2}}( \hat{b}^i_t - \hat{a}^i_t )).
\end{equation}
\end{itemize}

\begin{theorem}
For general $s$ as outlined above, policy $\pi^*_{\text{U,s}}$ as defined above is asymptotically optimal.  In particular, for any choice of $\{ f_i = f_{a_i, b_i} \} \subset \mathcal{F}$, with $s^* = \max_i s(a_i, b_i)$, for each sub-optimal bandit $i$ the following holds:
\begin{equation}\label{eqn:uniform-limit}
\lim_n \frac{ \mathbb{E}\left[ T^i_{\pi^*_{\text{U,s}}}(n) \right] }{ \ln n } = \frac{ 1 }{ \min_{ b_i \leq b } \left\{ \ln\left( b - a_i  \right) : s(a_i, b) \geq s^* \right\} - \ln( b_i - a_i) }.
\end{equation}
Taking the particular choice of $s_\mu(a, b) = (a + b)/2$, this yields for all sub-optimal $i$,
\begin{equation}\label{eqn:uniform-limit-mu}
\lim_n \frac{ \mathbb{E}\left[ T^i_{\pi^*_{\text{U},s_\mu}}(n) \right] }{ \ln n } = \frac{ 1 }{ \ln\left( \frac{2s^* - 2a_i}{b_i - a_i}  \right) }.
\end{equation}
\end{theorem}

\begin{proof}
To begin, it can be shown that
\begin{equation}\label{eqn:a-and-i-uniform}
\begin{split}
\mathbf{I}( f_{a, b}, f_{\tilde{a}, \tilde{b}} ) & =
\begin{cases}
\ln\left( \frac{\tilde{b} - \tilde{a}}{b-a}  \right) &\mbox{if } \tilde{a} \leq a, b \leq \tilde{b} \\ 
\infty & \mbox{else.}
\end{cases}\\
\mathbb{M}_{f_{a, b}}(\rho) & =\min_{ b \leq \tilde{b} } \left\{ \ln\left( \tilde{b} - a  \right) : s(a, \tilde{b}) \geq \rho \right\} - \ln( b - a).
\end{split}
\end{equation}
At this point, \Assumption B1 is easy to verify given the structure of the score function and the parameterization of $\mathcal{F}$. Note then that if $\mathbf{I}( f_{a, b}, f_{\tilde{a}, \tilde{b}} ) < \delta$, it follows that
\begin{equation}
\begin{split}
\tilde{a} & \leq a \\
b & \leq \tilde{b} \\
\tilde{b} - \tilde{a} & < (b-a)e^{\delta}.
\end{split}
\end{equation}
It follows that $0 \leq \tilde{b} - b < (b-a)(e^{\delta}-1)$ and $0 \leq a - \tilde{a} < (b-a)(e^{\delta}-1)$. From this, we may conclude that any function of $f \in \mathcal{F}$ that is continuous with respect to the parameters is continuous with respect to $f$ under $\mathbf{I}$. This verifies \Assumption B2. Given the above considerations, for verifying \Assumptions R1 - R3, it is convenient to take $\mathbf{\nu} = \mathbf{I}$ as the similarity measure on $\mathcal{F}$. Note that the continuity of $s$ with respect to $b$ makes $\mathbb{M}_f(\rho)$ continuous with respect to $\rho$. This, and the above considerations, verifies \Assumption R1.

To verify \Assumption R2, note that $a \leq \hat{a}_t \leq \hat{b}_t \leq b$. Hence, we have the following:
\begin{equation}
\begin{split}
\mathbb{P}\left( \mathbf{I}( \hat{f}_t, f ) > \delta \right) & = \mathbb{P}\left( (b-a) > (\hat{b}_t-\hat{a}_t)e^{\delta} \right) = \mathbb{P}\left( e^{-\delta} > \frac{\hat{b}_t-\hat{a}_t}{b-a} \right)
\end{split}
\end{equation}
Here, we utilize the following Lemma, characterizing the distribution of $\hat{a}_t, \hat{b}_t$:
\begin{lemma}\label{lem:uniform}
For $t \geq 2, 0 < \lambda < 1$:
\begin{equation}
\mathbb{P}\left(  \frac{\hat{b}_t-\hat{a}_t}{b-a} < \lambda \right) = \left( t(1-\lambda) + \lambda \right) \lambda^{t-1} \leq (t + 1)\lambda^{t-1}.
\end{equation}
\end{lemma}
The proof is given in the Appendix.
Hence we see that
\begin{equation}
\mathbb{P}\left( \mathbf{I}( \hat{f}_t, f ) > \delta \right) \leq (t + 1)e^{-\delta(t-1)} = e^{-O(t)},
\end{equation}
verifying \Assumption R2.

For \Assumption R3, note that
\begin{equation}
\begin{split}
\mathbb{P}( \delta < \mathbb{M}_{ \hat{f}_t }(  \rho ) ) & = \mathbb{P}\left( \delta < \min_{ \hat{b}_t \leq \tilde{b} } \left\{ \ln\left( \frac{ \tilde{b} - \hat{a}_t }{ \hat{b}_t - \hat{a}_t }  \right) : s(\hat{a}_t, \tilde{b}) > \rho \right\} \right) \\
& = \mathbb{P}\left( \max_{ \hat{b}_t \leq \tilde{b} } \left\{ s( \hat{a}_t, \tilde{b} ) :  \ln\left( \frac{ \tilde{b} - \hat{a}_t }{ \hat{b}_t - \hat{a}_t }  \right) \leq \delta \right\} < \rho \right) \\
& = \mathbb{P}\left( \max_{ \hat{b}_t \leq \tilde{b} } \left\{ s( \hat{a}_t, \tilde{b} ) :  \tilde{b} \leq \hat{a}_t + e^{\delta}(\hat{b}_t - \hat{a}_t) \right\} < \rho \right) \\
& = \mathbb{P}\left(  s( \hat{a}_t, \hat{a}_t + e^{\delta}(\hat{b}_t - \hat{a}_t)) < \rho \right) \\
& \leq \mathbb{P}\left(  s( a, a + e^{\delta}(\hat{b}_t - \hat{a}_t)) < \rho \right).
\end{split}
\end{equation}
Hence we have that
\begin{equation}
\mathbb{P}( \delta < \mathbb{M}_{ \hat{f}_t }(  s(f) - \epsilon ) ) \leq \mathbb{P}\left(  s( a, a + e^{\delta}(\hat{b}_t - \hat{a}_t)) < s(a,b) - \epsilon \right).
\end{equation}
Given the continuity of $s$, let $\tilde{\epsilon} > 0$ be such that $s(a, b- \tilde{\epsilon}) \geq s(a,b) - \epsilon$.
\begin{equation}
\begin{split}
\mathbb{P}( \delta < \mathbb{M}_{ \hat{f}_t }(  s(f) - \epsilon ) ) &  \leq \mathbb{P}\left(  s( a, a + e^{\delta}(\hat{b}_t - \hat{a}_t)) < s(a,b - \tilde{\epsilon}) \right) \\
& = \mathbb{P}\left(  a + e^{\delta}(\hat{b}_t - \hat{a}_t) < b - \tilde{\epsilon} \right) \\
& = \mathbb{P}\left( \frac{\hat{b}_t - \hat{a}_t}{b-a} < e^{-\delta}\left(1 - \frac{\tilde{\epsilon}}{b-a}\right) \right) \\
& \leq (t+1)e^{-\delta(t-1)}\left( 1 - \frac{\tilde{\epsilon}}{b-a}\right)^{t-1} = e^{-\Omega(t)}e^{-\delta(t-1)}.
\end{split}
\end{equation}
This verifies \Assumption R3, with $d_t = 1$, producing a bound of the correct order. This in turn verifies the policy as optimal, taking $\tilde{d}(t) = 2$, and Eq. \ref{eqn:uniform-limit} follows from the definition of $\mathbb{M}_f(\rho)$ for this model.\qed
\end{proof}

\section{Three Examples of Normal Bandits}\label{sec:normal}
In this section, we consider the case of the bandits being chosen from a set or sets of normal densities, with $f_{\mu, \sigma}(x) = e^{-(x-\mu)^2/(2\sigma^2)}/(\sigma \sqrt{ 2 \pi})$. In the three examples discussed, the family or families of potential distributions will be restricted in certain ways, but the following general discussion relative to normal distributions is useful. In particular, for a general normal density $f = f_{\mu, \sigma}$, given $t$ many i.i.d. samples from $f$, we take $\hat{f}_t = f_{\hat{\mu}_t, \hat{\sigma}_t}$ where
\begin{equation}
\begin{split}
\hat{\mu}_t & = \frac{1}{t} \sum_{n = 1}^t X_n,\\
\hat{\sigma}^2_t & = \frac{1}{t-1} \sum_{n = 1}^t \left( X_n - \hat{\mu}_t \right)^2.
\end{split}
\end{equation}
Recall the classic result, that $(\hat{\mu}_t - \mu) \sqrt{t}/\sigma$ and $\hat{\sigma}^2_t (t-1)/\sigma^2$ are independent, with a standard normal and a $\chi^2_{t-1}$ distribution, respectively. The following lemma, proven in the Appendix, will prove useful:
\begin{lemma}\label{lem:normal-bound}
Let $U_t \sim \chi^2_t$, and $Z$ be a standard normal. For $z \geq 0$, and $0 < u^{-} < 1 < u^{+} < \infty$, the following bounds hold:
\begin{equation}
\begin{split}
\mathbb{P}\left( U_t > u^{+} t\right) & \leq \left( u^{+} e^{1 - u^{+}} \right)^{\frac{t}{2}}\\
\mathbb{P}\left( U_t < u^{-} t\right) & \leq  \left( u^{-} e^{1-u^{-}} \right)^{\frac{t}{2}}\\
\mathbb{P}\left( Z > z \right) & \leq \frac{1}{2}e^{-z^2/2}.
\end{split}
\end{equation}
\end{lemma}

Since the domain of any such distribution is the whole of $\mathbb{R}$, it is not difficult to show that for any normal densities $f, g$:
\begin{equation}
\mathbf{I}(f, g) = \frac{ (\mu_f - \mu_g)^2 }{2 \sigma_g^2 } + \frac{1}{2} \left( \frac{ \sigma_f^2 }{ \sigma_g^2 } - \ln \left( \frac{ \sigma_f^2 }{ \sigma_g^2 } \right) - 1 \right).
\end{equation}

Again, let $L^-(\delta)$ and $L^+(\delta)$ be the smallest and largest positive solutions to $L - \ln L - 1 = \delta$, respectively. Note that if $\mathbf{I}(f,g) < \delta$, it follows that
\begin{equation}
\begin{split}
\frac{ (\mu_f - \mu_g)^2 }{2 \sigma_g^2 } & < \delta, \\
\frac{1}{2} \left( \frac{ \sigma_f^2 }{ \sigma_g^2 } - \ln \left( \frac{ \sigma_f^2 }{ \sigma_g^2 } \right) - 1 \right) & < \delta.
\end{split}
\end{equation}
From the above, we have that $\sigma_f^2 / L^+(2\delta) < \sigma_g^2 < \sigma_f^2/L^-(2\delta)$ and
\begin{equation}
\lvert \mu_f - \mu_g \rvert < \sigma_g \sqrt{2\delta} < \sigma_f \sqrt{2 \delta / L^-(2\delta) }.
\end{equation}
Since $L^{\pm}(2\delta) \to 1$ and $\delta / L^-(2\delta) \to 0$ as $\delta \to 0$, the above implies that any functional of normal densities that is a continuous function of the parameters of $f$ over the family of densities will be continuous with respect to $f$ under $\mathbf{I}$.

\subsection{Unknown Means and Unknown Variances: Maximizing Expected Value}
In this section, we take $\mathcal{F}$ as the family of unrestricted normal distributions:
\begin{equation}
\mathcal{F} = \left\{ f_{\mu,\sigma}(x) = \frac{1}{\sigma \sqrt{2\pi} } e^{ -\frac{1}{2\sigma^2}( x - \mu )^2 } : -\infty < \mu < \infty, \sigma > 0 \right\}.
\end{equation}
As such, this section essentially reproduces the   result of \cc{chk2015} (presented therein in terms of classical regret) in the framework established herein.
In this  case the controller is interested in activating the bandit with maximum expected value   as often as possible. This can be achieved if we
 take the score functional of interest here  to be  the expected value, $$s(f) = \int_\mathbb{R} x f(x) dx =\mu .$$  
 
  We define the specific instance of policy $\pi^*$ under this model:

\textbf{Policy $\boldmath \pi_{\text{CHK}}$ (UCB-NORMAL)}
\begin{itemize}
\item[i)] For $n = 1, 2, \ldots 3N$, sample each bandit $3$ times, and
\item[ii)] for $n \geq 3N$, sample from bandit $\pi_{\text{CHK}}(n+1) = \argmax_i u_i\left( n, T^i_{\pi_\text{CHK}}(n) \right)$ breaking ties uniformly at random, where
\begin{equation}
u_i(n, t) = \hat{\mu}^i_t + \hat{\sigma}^i_t \sqrt{ n^{\frac{2}{t-2}} - 1 }.
\end{equation}
\end{itemize}

\begin{theorem}
For $s(f_{\mu, \sigma}) = \mu$ in the above model, policy $\pi_{\text{CHK}}$ as defined above is asymptotically optimal. In particular, for any choice of $\{ f_i = f_{\mu_i, \sigma_i} \} \subset \mathcal{F}$, with $\mu^* = \max_i \mu_i$, for each sub-optimal bandit $i$ the following holds:
\begin{equation}\label{eqn:normal-limit}
\lim_n \frac{ \mathbb{E}\left[ T^i_{\pi_{\text{CHK}}}(n) \right] }{ \ln n } = \frac{ 2 }{ \ln \left( 1 + \frac{ (\mu^* - \mu_i)^2 }{ \sigma_i^2 } \right) }.
\end{equation}
\end{theorem}

\begin{proof}
\Assumption B1 is easy to verify given the parameterization of $\mathcal{F}$. As already established, any score functional $s(f)$ that is continuous with respect to the parameters of $f$ is continuous with respect to $f$ under $\mathbf{I}$. Taking $s(f_{\mu, \sigma}) = \mu$, this verifies \Assumption B2. Further, given the formula for $\mathbf{I}(f,g)$ above, we have that
\begin{equation}
\mathbb{M}_{f_{\mu, \sigma}}(\rho) = 
\begin{cases}  \frac{1}{2} \ln \left( 1 + \frac{ (\rho - \mu)^2 }{ \sigma^2 } \right) &\mbox{if } \rho > \mu \\ 
0 & \mbox{else. } \end{cases}
\end{equation}
In verifying the R-\Assumptions, we take as similarity measure $\mathbf{\nu} = \mathbf{I}$. By the previous commentary, $\mathbb{M}_{f_{\mu, \sigma}}(\rho)$ is continuous with respect to $f$ under $\mathbf{I}$, as well as being continuous with respect to $\rho$, by inspection. This verifies \Assumption R1. To verify \Assumption R2, observe the following:
\begin{equation}
\begin{split}
\mathbb{P}\left( \mathbf{I}( \hat{f}_t, f) > \delta \right) & = \mathbb{P}\left( \frac{ (\hat{\mu}_t - \mu)^2 }{2 \sigma^2 } + \frac{1}{2} \left( \frac{ \hat{\sigma}_t^2 }{ \sigma^2 } - \ln \left( \frac{ \hat{\sigma}_t^2 }{ \sigma^2 } \right) - 1 \right) > \delta \right) \\
& \leq  \mathbb{P}\left( \frac{ (\hat{\mu}_t - \mu)^2 }{\sigma^2 } > \delta \right) +  \mathbb{P}\left(  \frac{ \hat{\sigma}_t^2 }{ \sigma^2 } - \ln \left( \frac{ \hat{\sigma}_t^2 }{ \sigma^2 } \right) - 1 > \delta \right) \\
& = \mathbb{P}\left( \frac{ (\hat{\mu}_t - \mu)^2 }{\sigma^2 } > \delta \right) +  \mathbb{P}\left( \frac{ \hat{\sigma}_t^2 }{ \sigma^2 } < L^{-}(\delta ) \right) + \mathbb{P}\left( \frac{ \hat{\sigma}_t^2 }{ \sigma^2 } > L^{+}(\delta) \right) \\
& = \mathbb{P}\left( Z^2 > \delta t  \right) +  \mathbb{P}\left( U_{t-1} < (t-1) L^{-}(\delta ) \right) + \mathbb{P}\left( U_{t-1} > (t-1) L^{+}(\delta) \right),
\end{split}
\end{equation}
where $Z$ is a standard normal, and $U_{t-1} \sim \chi^2_{t-1}$. We may then apply Lemma \ref{lem:normal-bound} to bound the above. Taking $u^{\pm} = L^{\pm}(\delta)$, we have $u^{\pm} e^{1 - u^{\pm}} = e^{-\delta}$, as $L^{\pm}(\delta) - \ln L^{\pm}(\delta) - 1 = \delta$. Hence,
\begin{equation}
\begin{split}
\mathbb{P}\left( \mathbf{I}( \hat{f}_t, f) > \delta \right) & \leq 2\mathbb{P}\left( Z > \sqrt{ \delta t }  \right) +  e^{-\delta\frac{t-1}{2}} + e^{-\delta\frac{t-1}{2}} \\
& \leq e^{-\frac{1}{2} \delta t} + 2  e^{-\delta\frac{t-1}{2}} = (2 e^{\delta/2} + 1)e^{-\frac{t \delta}{2}} = e^{-O(t)}.
\end{split}
\end{equation}
This verifies \Assumption R2.

For \Assumption R3, note that
\begin{equation}
\begin{split}
\mathbb{P}( \delta < \mathbb{M}_{ \hat{f}_t }(  \rho ) ) & = \mathbb{P}\left( \delta < \frac{1}{2} \ln \left( 1 + \frac{ (\rho - \hat{\mu}_t)^2 }{ \hat{\sigma}_t^2 } \right) \text{ and } \rho > \hat{\mu}_t \right) \\
& = \mathbb{P}\left(\hat{\sigma}_t  \sqrt{ e^{2\delta}-1 } < \lvert \rho - \hat{\mu}_t \rvert  \text{ and } \rho > \hat{\mu}_t \right) \\
& = \mathbb{P}\left( \hat{\mu}_t + \hat{\sigma}_t  \sqrt{ e^{2\delta}-1 } < \rho \right).
\end{split}
\end{equation}
Hence,
\begin{equation}
\begin{split}
\mathbb{P}( \delta < \mathbb{M}_{ \hat{f}_t }(  \mu-\epsilon ) ) & = \mathbb{P}\left( \hat{\mu}_t + \hat{\sigma}_t  \sqrt{ e^{2\delta}-1 } < \mu-\epsilon \right) \\
& = \mathbb{P}\left( Z \sigma/\sqrt{t} + \hat{\sigma}_t  \sqrt{ e^{2\delta}-1 } < -\epsilon \right) \\
& = \mathbb{P}\left( \frac{\epsilon}{\sigma}\sqrt{t}  + \frac{\hat{\sigma}_t}{\sigma} \sqrt{t}  \sqrt{ e^{2\delta}-1 } < Z \right) \\
& \leq \frac{1}{2} \mathbb{E}\left[ e^{-\frac{1}{2}\left(\frac{\epsilon}{\sigma}\sqrt{t}  + \frac{\hat{\sigma}_t}{\sigma} \sqrt{t}  \sqrt{ e^{2\delta}-1 }\right)^2} \right] \\
& \leq \frac{1}{2} e^{-\frac{1}{2}\frac{\epsilon^2}{\sigma^2}t} \mathbb{E}\left[ e^{-\frac{1}{2} \frac{\hat{\sigma}_t^2}{\sigma^2} t  \left( e^{2\delta}-1 \right)} \right] \\
& = \frac{1}{2} e^{-\frac{1}{2}\frac{\epsilon^2}{\sigma^2}t} \mathbb{E}\left[ e^{-\frac{1}{2} U_{t-1} \frac{t}{t-1}  \left( e^{2\delta}-1 \right)} \right] \\
& = \frac{1}{2} e^{-\frac{1}{2}\frac{\epsilon^2}{\sigma^2}t} \left( \frac{t-1}{ e^{2\delta}t - 1}\right)^{ \frac{t-1}{2} } \\
& \leq \frac{1}{2} e^{-\frac{1}{2}\frac{\epsilon^2}{\sigma^2}t} e^{-\delta(t-1)}.
\end{split}
\end{equation}
The last step follows, as taking $\delta > 0$. This verifies \Assumption R3, with $d_t = 1$, producing a bound of the correct order. This in turn verifies the policy as optimal, taking $\tilde{d}(t) = 2$.
\qed
\end{proof}

\subsection{Equal Means and Unknown Variances: Minimizing Variance}
In this section, we consider $\mathcal{F}$ as the family of normal distributions, each with equal mean $\mu$. 
\begin{equation}
\mathcal{F} = \left\{ f_{\mu,\sigma}(x) = \frac{1}{\sigma \sqrt{2\pi} } e^{ -\frac{1}{2\sigma^2}( x - \mu )^2 } : \sigma > 0 \right\}.
\end{equation}
We take a slight departure from the previous examples in the following way: that in all previous cases, it was assumed that the controller had complete knowledge of $\mathcal{F}$. In this case, we assume that the controller knows that $\mathcal{F}$ is a family of normal distributions, and that every distribution in $\mathcal{F}$ has the same mean, but we assume the specific value of that mean, $\mu$, is unknown to the controller. It is interesting that all relevant computations still go through. Note, for instance, that in this case, independent of $\mu$, for $f, g \in \mathcal{F}$:
\begin{equation}
\mathbf{I}(f, g) =\frac{1}{2} \left( \frac{ \sigma_f^2 }{ \sigma_g^2 } - \ln \left( \frac{ \sigma_f^2 }{ \sigma_g^2 } \right) - 1 \right).
\end{equation}

In this  case the controller is interested in activating the bandit with minimal variance as often as possible. This can be achieved if we
 take the score functional of interest here to be the inverse of the variance, i.e., $$s(f)=s(\mu, \sigma) = 1/\sigma^2.$$ 
 
 This models, for instance, each bandit the controller is faced with as a process for achieving some desired goal or output, and the controller wanting to constrain the output as much as possible. In this section, the estimators $\hat{f}_t$ are understood to be $f_{\mu, \hat{\sigma}_t}$, though the presence of $\mu$ is mainly symbolic as it is unknown and, as will be shown, unnecessary.

We define the specific instance of policy $\pi^*$ under this model:

\textbf{Policy $\boldmath \pi^*_{\sigma}$ (UCB-NORMAL-VARIANCE)}
\begin{itemize}
\item[i)] For $n = 1, 2, \ldots 3N$, sample each bandit $3$ times, and
\item[ii)] for $n \geq 3N$, sample from bandit $\pi^*_{\sigma}(n+1) = \argmax_i u_i\left( n, T^i_{\pi^*_{\sigma}}(n) \right)$ breaking ties uniformly at random, where
\begin{equation}
u_i(n, t) =  (\hat{\sigma}^i_t)^{-2} L^{+}\left( \frac{ 2\ln n }{t - 2} \right),
\end{equation}
again taking $L^+(\delta)$ as the largest positive solution to $L - \ln L - 1 = \delta$.
\end{itemize}

\begin{theorem}
For $s(f) = 1/\text{Var}_f(X)$ in the above model, policy $\pi^*_{\sigma}$ as defined above is asymptotically optimal. In particular, for any choice of $\{ f_i = f_{\mu, \sigma_i} \} \subset \mathcal{F}$, with $\sigma^* = \min_i \sigma_i$, for each sub-optimal bandit $i$ the following holds:
\begin{equation}\label{eqn:variance-limit}
\lim_n \frac{ \mathbb{E}\left[ T^i_{\pi^*_{\sigma}}(n) \right] }{ \ln n } = \frac{ 2 }{  \frac{\sigma_i^2}{{\sigma^*}^2} - \ln\left( \frac{\sigma_i^2}{{\sigma^*}^2}\right) - 1 }.
\end{equation}
\end{theorem}

Before giving the proof, we note the following observation: The estimator utilized here to estimate $\sigma^2$ depends explicitly on the estimator $\hat{\mu}_t$ for $\mu$. While the above policy is asymptotically optimal, finite horizon improvements could be achieved for instance estimating the variance by utilizing either $\mu$ explicitly as the mean, in the case of known mean, or by utilizing all samples from all bandits simultaneously to estimate the mean, in the case of unknown but known to be equivalent mean.

\begin{proof}
\Assumption B1 is easy to verify given the parameterization of $\mathcal{F}$, as is \Assumption B2 given the previous discussion of continuity under $\mathbf{I}$. Further, given the formula for $\mathbf{I}(f,g)$ above, we have that
\begin{equation}
\mathbb{M}_{f_{\mu, \sigma}}(\rho) = 
\begin{cases}  \frac{1}{2}\left( \rho \sigma^2 - \ln( \rho \sigma^2 ) - 1\right) &\mbox{if } \rho > 1/\sigma^2 \\ 
0 & \mbox{else. } \end{cases}
\end{equation}
It is again convenient to take $\mathbf{\nu} = \mathbf{I}$ in verifying the R-\Assumptions. By the previous commentary, $\mathbb{M}_{f_{\mu, \sigma}}(\rho)$ is continuous with respect to $f$ under $\mathbf{I}$, as well as being continuous with respect to $\rho$, by inspection. This verifies \Assumption R1. To verify \Assumption R2, observe the following:
\begin{equation}
\begin{split}
\mathbb{P}\left( \mathbf{I}( \hat{f}_t, f) > \delta \right) & = \mathbb{P}\left( \frac{1}{2} \left( \frac{ \hat{\sigma}_t^2 }{ \sigma^2 } - \ln \left( \frac{ \hat{\sigma}_t^2 }{ \sigma^2 } \right) - 1 \right) > \delta \right) \\
& \leq  \mathbb{P}\left(  \frac{ \hat{\sigma}_t^2 }{ \sigma^2 } - \ln \left( \frac{ \hat{\sigma}_t^2 }{ \sigma^2 } \right) - 1 > 2\delta \right) \\
& =  \mathbb{P}\left( \frac{ \hat{\sigma}_t^2 }{ \sigma^2 } < L^{-}(2\delta ) \right) + \mathbb{P}\left( \frac{ \hat{\sigma}_t^2 }{ \sigma^2 } > L^{+}(2\delta) \right) \\
& = \mathbb{P}\left( U_{t-1} < (t-1) L^{-}(2\delta ) \right) + \mathbb{P}\left( U_{t-1} > (t-1) L^{+}(2\delta) \right),
\end{split}
\end{equation}
where again, $U_{t-1} \sim \chi^2_{t-1}$. We may then apply Lemma \ref{lem:normal-bound} to bound the above. Taking $u^{\pm} = L^{\pm}(2\delta)$, we have $u^{\pm} e^{1 - u^{\pm}} = e^{-2\delta}$, as $L^{\pm}(2\delta) - \ln L^{\pm}(2\delta) - 1 = 2\delta$. Hence,
\begin{equation}
\begin{split}
\mathbb{P}\left( \mathbf{I}( \hat{f}_t, f) > \delta \right) & \leq e^{-2\delta\frac{t-1}{2}} + e^{-2\delta\frac{t-1}{2}} = 2e^{-\delta(t-1)} = e^{-O(t)}.
\end{split}
\end{equation}
This verifies \Assumption R2.

For \Assumption R3, note that
\begin{equation}
\begin{split}
\mathbb{P}( \delta < \mathbb{M}_{ \hat{f}_t }(  \rho ) ) & = \mathbb{P}\left( \delta < \frac{1}{2}\left( \rho \hat{\sigma}_t^2 - \ln( \rho \hat{\sigma}_t^2 ) - 1\right) \text{ and } \rho > 1/\hat{\sigma}_t^2 \right)  = \mathbb{P}\left(  \rho \hat{\sigma}_t^2 > L^{+}(2\delta) \right).
\end{split}
\end{equation}
Let $1/\sigma^2 > \epsilon > 0$, and let $\tilde{\epsilon} = \epsilon \sigma^2$. Then,
\begin{equation}
\begin{split}
\mathbb{P}( \delta < \mathbb{M}_{ \hat{f}_t }(  1/\sigma^2 - \epsilon ) ) & = \mathbb{P}( \delta < \mathbb{M}_{ \hat{f}_t }(  1/\sigma^2(1 - \tilde{\epsilon}) ) ) \\
& =  \mathbb{P}\left( 1/\sigma^2(1 - \tilde{\epsilon}) \hat{\sigma}_t^2 > L^{+}(2\delta) \right) \\
& = \mathbb{P}\left( (1 - \tilde{\epsilon}) U_{t-1} > (t-1) L^{+}(2\delta) \right) \\
& \leq \left( \frac{ L^{+}(2\delta) }{1 - \tilde{\epsilon}} e^{1 - \frac{ L^{+}(2\delta) }{1 - \tilde{\epsilon}}} \right)^{\frac{t-1}{2}}.
\end{split}
\end{equation}
The last step is an application of Lemma \ref{lem:normal-bound}. Noting that $L^{+}(2\delta) = e^{ L^{+}(2\delta) - 2 \delta - 1}$, the above can be simplified to
\begin{equation}
\begin{split}
\mathbb{P}( \delta < \mathbb{M}_{ \hat{f}_t }(  1/\sigma^2 - \epsilon ) ) & \leq  \left( \frac{e^{- \frac{ L^{+}(2\delta) \tilde{\epsilon}}{1 - \tilde{\epsilon}} }}{1-\tilde{\epsilon}}\right)^{\frac{t-1}{2}}e^{-\delta(t-1)} \\
& \leq \left( \frac{e^{-\frac{\tilde{\epsilon}}{1 - \tilde{\epsilon}} }}{1-\tilde{\epsilon}}\right)^{\frac{t-1}{2}}e^{-\delta(t-1)} = e^{-\Omega(t)}e^{-\delta(t-1)}.
\end{split}
\end{equation}
The penultimate bound follows, as $L^{+}(2\delta) \geq 1$. This verifies \Assumption R3, with $d_t = 1$, producing a bound of the correct order. This in turn verifies the policy as optimal, taking $\tilde{d}(t) = 2$.
\qed
\end{proof}

\subsection{A Heterogeneous Normal Model}
As an example of the heterogenous bandit model presented in Section \ref{sec:weakening}, consider the following model: for each $i = 1, \ldots, N$, let $\sigma_i > 0$ be known, and define:
\begin{equation}
\mathcal{F}_i = \left\{ f_{\mu, \sigma_i}(x) = \frac{1}{\sigma_i \sqrt{2\pi} } e^{ -\frac{1}{2\sigma_i^2}( x - \mu )^2 } : -\infty < \mu < \infty \right\}.
\end{equation}
This models the case that for each bandit $i$, the controller knows the bandit has a normal distribution, with known variance $\sigma_i^2$, but with unknown mean.

The focus of this section 
is  the case in which the controller, given a     threshold value $\kappa$,  is interested in activating  bandits $i$ with the highest unknown   tail probability: 
$\int_{\kappa}^{\infty} f_i(x) dx=\mathbb{P}(X^i_k>k)$
 as  often as possible. This can be achieved if we 
 take the score functional of interest here to be  
$$s_\kappa(f) = \int_{\kappa}^{\infty} f(x) dx.$$ 
Taking $\Phi$ as the c.d.f. of a standard normal, and noting that in this model
$f$ is specified by $(\mu,\sigma)$   the above score function can be written
as  
\begin{equation}
s_\kappa(\mu, \sigma) = 1 - \Phi\left( \frac{ \kappa - \mu }{ \sigma } \right).
\end{equation}

It is easy to show in this case that for $f_i=f_{\mu_{f_i},\sigma_i},$ and $ g=g_{\mu_{g_i},\sigma_i} \in \mathcal{F}_i$:
\begin{equation}
\mathbf{I}(f_i, g_i) = \frac{ (\mu_{f_i} - \mu_{g_i})^2 }{2 \sigma_i^2 }.
\end{equation}
Note then that for fixed $i$, for $f, g \in \mathcal{F}_i$, if $\mathbf{I}(f,g) < \delta$, then $\lvert \mu_f - \mu_g \rvert < \sigma_i \sqrt{2 \delta}$. It follows easily from this that any score functional $s_i(f)$ that is a continuous function of the parameter of $f$, the mean, is continuous in $\mathcal{F}_i$ with respect to $f$ under $\mathbf{I}$.

For $f_i = f_{\mu_i, \sigma_i} \in \mathcal{F}_i$, and  $t$  samples under $f$, we take the estimator $\hat{f}^i_t = f_{\hat{\mu}^i_t,\sigma_i} \in \mathcal{F}_i$ where
\begin{equation}
\begin{split}
\hat{\mu}^i_t & = \frac{1}{t} \sum_{n = 1}^t X_n.
\end{split}
\end{equation}
Note that  $\hat{\mu}^i_t$ is normally distributed with mean $\mu_i$ and variance $\sigma_i^2/t$.

We next define the specific instance of policy $\pi^*$ under this model:

\textbf{Policy $\boldmath \pi^*_{\kappa}$ (UCB-NORMAL-THRESHOLD)}
\begin{itemize}
\item[i)] For $n = 1, 2, \ldots 2N$, sample each bandit $2$ times, and
\item[ii)] for $n \geq 2N$, sample from bandit $\pi^*_{\kappa}(n+1) = \argmax_i u_i\left( n, T^i_{\pi^*_\kappa}(n) \right)$ breaking ties uniformly at random, where
\begin{equation}
u_i(n, t) = 1 - \Phi\left( \frac{ \kappa - \hat{\mu}^i_t }{ \sigma_i } - \sqrt{ \frac{ 2\ln n }{t-1} }  \right).
\end{equation}
\end{itemize}

\begin{theorem}
For $s_\kappa(f) = \mathbb{P}_f( X \geq \kappa )$ in the above model, policy $\pi^*_{\kappa}$ as defined above is asymptotically optimal. In particular, for any choice of $( f_i = f_{\mu_i, \sigma_i} )_{i  = 1}^N \in \bigotimes_{i = 1}^N \mathcal{F}_i$, with $s^* = \max_i s_\kappa(f_i)$, for each sub-optimal bandit $i$ the following holds:
\begin{equation}\label{eqn:normal-limit}
\lim_n \frac{ \mathbb{E}\left[ T^i_{\pi^*_{\kappa}}(n) \right] }{ \ln n } = \frac{2}{\left( \frac{\kappa - \mu_i}{\sigma_i} -\Phi^{-1}\left( 1 - s^* \right) \right)^2}\,.
\end{equation}
\end{theorem}

\begin{proof}
\Assumption $\text{B1}^\prime$ is easy to verify given the parameterization of the $\mathcal{F}_i$. As already established, any score functional $s(f)$ that is continuous with respect to the parameters of $f$ is continuous with respect to $f$ under $\mathbf{I}$. Taking $s_\kappa(f)$ as above verifies \Assumption $\text{B2}^\prime$. Further, given the formula for $\mathbf{I}(f,g)$ above, we have that for each $i$,
\begin{equation}
\mathbb{M}^i_{f_{\mu, \sigma_i}}(\rho) = 
\begin{cases} 0 &\mbox{if } 1 - \rho > \Phi\left( \frac{\kappa - \mu}{\sigma_i} \right)  \\ 
 \frac{1}{2} \left( \frac{ \kappa - \mu}{\sigma_i} - \Phi^{-1}\left(1 - \rho\right) \right)^2 & \mbox{else. } \end{cases}
\end{equation}
Again, for the purpose of verifying the R-\Assumptions, we take $\mathbf{\nu} = \mathbf{I}$. By the previous commentary, $\mathbb{M}^i_{f_{\mu, \sigma_i}}(\rho)$ is continuous with respect to $f$ under $\mathbf{I}$, as well as being continuous with respect to $\rho$, by inspection. This verifies \Assumption $\text{R1}^\prime$. To verify \Assumption $\text{R2}^\prime$, observe the following, that for each $i$:
\begin{equation}
\mathbb{P}\left( \mathbf{I}( \hat{f}_t, f) > \delta \right) = \mathbb{P}\left( \frac{ (\hat{\mu}_t - \mu)^2 }{2 \sigma_i^2 } > \delta \right) = \mathbb{P}\left( Z^2 > 2\delta t  \right) \leq \frac{1}{2}e^{-\delta t},
\end{equation}
taking $Z$ as a standard normal.

For \Assumption $\text{R3}^\prime$, note that
\begin{equation}
\begin{split}
\mathbb{P}( \delta < \mathbb{M}^i_{ \hat{f}_t }(  \rho ) ) & = \mathbb{P}\left( \delta <  \frac{1}{2} \left( \frac{ \kappa - \hat{\mu}_t}{\sigma_i} - \Phi^{-1}\left(1 - \rho\right) \right)^2 \text{ and } 1 - \rho \leq \Phi\left( \frac{\kappa - \mu}{\sigma_i} \right) \right) \\
& =  \mathbb{P}\left(  \sqrt{ 2\delta } <  \frac{\kappa - \hat{\mu}_t}{ \sigma_i } - \Phi^{-1}\left(1 - \rho\right)  \right) \\
& =  \mathbb{P}\left(  \sqrt{ 2\delta } < \frac{\kappa - \mu}{ \sigma_i } - \frac{Z}{ \sqrt{t} } - \Phi^{-1}\left(1 - \rho\right)  \right) \\
& =  \mathbb{P}\left( \left( \Phi^{-1}\left(1 - \rho\right) - \frac{\kappa - \mu}{ \sigma_i } +  \sqrt{ 2\delta } \right) \sqrt{t} < Z \right),
\end{split}
\end{equation}
where $Z$ is a standard normal random variable. Taking $\rho = s_\kappa(f) - \epsilon = 1 - \Phi( (\kappa - \mu)/\sigma_i ) - \epsilon$ in the above, note that $$\Phi^{-1}\left(1 - \rho\right) = \Phi^{-1}\left( \Phi\left( \frac{\kappa - \mu}{\sigma_i }\right) + \epsilon \right) >  \frac{\kappa - \mu}{ \sigma_i }.$$
Given this, let $\Delta = \Phi^{-1}\left(1 - \rho\right) - (\kappa - \mu)/\sigma_i > 0$. From the above, we have
\begin{equation}
\begin{split}
\mathbb{P}( \delta < \mathbb{M}^i_{ \hat{f}_t }(  s(f) - \epsilon ) ) & = \mathbb{P}\left( \left( \Delta +  \sqrt{ 2\delta } \right) \sqrt{t} < Z \right) \\
&  \leq \frac{1}{2} e^{-(\Delta + \sqrt{2\delta})^2t/2} \\
& \leq \frac{1}{2} e^{-\frac{1}{2}\Delta^2 t}e^{- \delta t }.
\end{split}
\end{equation}
This verifies \Assumption $\text{R3}^\prime$, with $d_t = 0$, producing a bound of the correct order. This in turn verifies the policy as optimal, taking $\tilde{d}(t) = 1$.
\qed
\end{proof}

\bibliographystyle{te}
\vspace{-.6cm}
\bibliography{mab2015,books,papers}

\begin{thebibliography}{34}
\newcommand{\enquote}[1]{``#1''}
\providecommand{\natexlab}[1]{#1}
\providecommand{\url}[1]{\texttt{#1}}
\providecommand{\urlprefix}{URL }
\providecommand{\bibAnnoteFile}[1]{%
  \IfFileExists{#1}{\begin{quotation}\noindent\textsc{Key:} #1\\
  \textsc{Annotation:}\ \input{#1}\end{quotation}}{}}
\providecommand{\bibAnnote}[2]{%
  \begin{quotation}\noindent\textsc{Key:} #1\\
  \textsc{Annotation:}\ #2\end{quotation}}

\bibitem[{Audibert et~al.(2009)Audibert, Munos, and
  Szepesv{\'a}ri}]{audibert2009exploration}
Audibert, Jean-Yves, R{\'e}mi Munos, and Csaba Szepesv{\'a}ri (2009),
  \enquote{Exploration - exploitation tradeoff using variance estimates in
  multi-armed bandits.} \emph{Theoretical Computer Science}, 410, 1876 -- 1902.
\bibAnnoteFile{audibert2009exploration}

\bibitem[{Auer and Ortner(2010)}]{auer2010ucb}
Auer, Peter and Ronald Ortner (2010), \enquote{Ucb revisited: Improved regret
  bounds for the stochastic multi-armed bandit problem.} \emph{Periodica
  Mathematica Hungarica}, 61, 55 -- 65.
\bibAnnoteFile{auer2010ucb}

\bibitem[{Bartlett and Tewari(2009)}]{bartlett2009regal}
Bartlett, Peter~L and Ambuj Tewari (2009), \enquote{Regal: A regularization
  based algorithm for reinforcement learning in weakly communicating mdps.} In
  \emph{Proceedings of the Twenty-Fifth Conference on Uncertainty in Artificial
  Intelligence}, 35 -- 42, AUAI Press.
\bibAnnoteFile{bartlett2009regal}

\bibitem[{Bubeck and Slivkins(2012)}]{bubeck2012best}
Bubeck, S{\'e}bastien and Aleksandrs Slivkins (2012), \enquote{The best of both
  worlds: Stochastic and adversarial bandits.} arXiv preprint arXiv:1202.4473.
\bibAnnoteFile{bubeck2012best}

\bibitem[{Burnetas et~al.(2015)Burnetas, Kanavetas, and Katehakis}]{bkk2015s}
Burnetas, Apostolos~N., Odysseas Kanavetas, and Michael~N. Katehakis (2015),
  \enquote{Asymptotically optimal multi-armed bandit policies under a cost
  constraint.} \emph{arXiv preprint arXiv:1509.02857}.
\bibAnnoteFile{bkk2015s}

\bibitem[{Burnetas and Katehakis(1993)}]{burnetas1993sequencing}
Burnetas, Apostolos~N and Michael~N Katehakis (1993), \enquote{On sequencing
  two types of tasks on a single processor under incomplete information.}
  \emph{Probability in the Engineering and Informational Sciences}, 7, 85 --
  119.
\bibAnnoteFile{burnetas1993sequencing}

\bibitem[{Burnetas and Katehakis(1996{\natexlab{a}})}]{BKlarge1996}
Burnetas, Apostolos~N and Michael~N Katehakis (1996{\natexlab{a}}), \enquote{On
  large deviations properties of sequential allocation problems.}
  \emph{Stochastic Analysis and Applications}, 14, 23 -- 31.
\bibAnnoteFile{BKlarge1996}

\bibitem[{Burnetas and Katehakis(1996{\natexlab{b}})}]{bkmab96}
Burnetas, Apostolos~N and Michael~N Katehakis (1996{\natexlab{b}}),
  \enquote{Optimal adaptive policies for sequential allocation problems.}
  \emph{Advances in Applied Mathematics}, 17, 122 -- 142.
\bibAnnoteFile{bkmab96}

\bibitem[{Burnetas and Katehakis(1997{\natexlab{a}})}]{burnetas1997finite}
Burnetas, Apostolos~N and Michael~N Katehakis (1997{\natexlab{a}}), \enquote{On
  the finite horizon one-armed bandit problem.} \emph{Stochastic Analysis and
  Applications}, 16, 845 -- 859.
\bibAnnoteFile{burnetas1997finite}

\bibitem[{Burnetas and Katehakis(1997{\natexlab{b}})}]{bkmdp97}
Burnetas, Apostolos~N and Michael~N Katehakis (1997{\natexlab{b}}),
  \enquote{Optimal adaptive policies for {M}arkov decision processes.}
  \emph{Mathematics of {O}perations {R}esearch}, 22, 222 -- 55.
\bibAnnoteFile{bkmdp97}

\bibitem[{Burnetas and Katehakis(2003)}]{burnetas2003asymptotic}
Burnetas, Apostolos~N and Michael~N Katehakis (2003), \enquote{Asymptotic
  {B}ayes analysis for the finite-horizon one-armed-bandit problem.}
  \emph{Probability in the Engineering and Informational Sciences}, 17, 53 --
  82.
\bibAnnoteFile{burnetas2003asymptotic}

\bibitem[{Butenko et~al.(2003)Butenko, Pardalos, and
  Murphey}]{butenko2003cooperative}
Butenko, Sergiy, Panos~M Pardalos, and Robert Murphey (2003), \emph{Cooperative
  Control: Models, Applications, and Algorithms}. Kluwer Academic Publishers.
\bibAnnoteFile{butenko2003cooperative}

\bibitem[{Capp{\'e} et~al.(2013)Capp{\'e}, Garivier, Maillard, Munos, and
  Stoltz}]{cappe2013kullback}
Capp{\'e}, Olivier, Aur{\'e}lien Garivier, Odalric-Ambrym Maillard, R{\'e}mi
  Munos, and Gilles Stoltz (2013), \enquote{Kullback - {L}eibler upper
  confidence bounds for optimal sequential allocation.} \emph{The Annals of
  Statistics}, 41, 1516 -- 1541.
\bibAnnoteFile{cappe2013kullback}

\bibitem[{Cowan et~al.(2015)Cowan, Honda, and Katehakis}]{chk2015}
Cowan, Wesley, Junya Honda, and Michael~N Katehakis (2015), \enquote{Asymptotic
  optimality, finite horizon regret bounds, and a solution to an open problem.}
  \emph{Journal of Machine Learning Research, to appear; preprint
  arXiv:1504.05823}.
\bibAnnoteFile{chk2015}

\bibitem[{Cowan and Katehakis(2015)}]{cowan2015multi}
Cowan, Wesley and Michael~N Katehakis (2015), \enquote{Multi-armed bandits
  under general depreciation and commitment.} \emph{Probability in the
  Engineering and Informational Sciences}, 29, 51 -- 76.
\bibAnnoteFile{cowan2015multi}

\bibitem[{Dayanik et~al.(2013)Dayanik, Powell, and
  Yamazaki}]{dayanik2013asymptotically}
Dayanik, Savas, Warren~B Powell, and Kazutoshi Yamazaki (2013),
  \enquote{Asymptotically optimal {B}ayesian sequential change detection and
  identification rules.} \emph{Annals of Operations Research}, 208, 337 -- 370.
\bibAnnoteFile{dayanik2013asymptotically}

\bibitem[{Denardo et~al.(2013)Denardo, Feinberg, and Rothblum}]{dena2013}
Denardo, Eric~V, Eugene~A Feinberg, and Uriel~G Rothblum (2013), \enquote{The
  multi-armed bandit, with constraints.} In \emph{Cyrus Derman Memorial Volume
  I: Optimization under Uncertainty: Costs, Risks and Revenues} (M.N.
  Katehakis, S.M. Ross, and J.~Yang, eds.), Annals of Operations Research,
  Springer, New York.
\bibAnnoteFile{dena2013}

\bibitem[{Feinberg et~al.(2014)Feinberg, Kasyanov, and
  Zgurovsky}]{feinberg2014convergence}
Feinberg, Eugene~A, Pavlo~O Kasyanov, and Michael~Z Zgurovsky (2014),
  \enquote{Convergence of value iterations for total-cost mdps and pomdps with
  general state and action sets.} In \emph{Adaptive Dynamic Programming and
  Reinforcement Learning (ADPRL), 2014 IEEE Symposium on}, 1 -- 8, IEEE.
\bibAnnoteFile{feinberg2014convergence}

\bibitem[{Filippi et~al.(2010)Filippi, Capp{\'e}, and
  Garivier}]{filippi2010optimism}
Filippi, Sarah, Olivier Capp{\'e}, and Aur{\'e}lien Garivier (2010),
  \enquote{Optimism in reinforcement learning based on {K}ullback {L}eibler
  divergence.} In \emph{48th Annual Allerton Conference on Communication,
  Control, and Computing}.
\bibAnnoteFile{filippi2010optimism}

\bibitem[{Gittins(1979)}]{gittins-79}
Gittins, John~C. (1979), \enquote{Bandit processes and dynamic allocation
  indices (with discussion).} \emph{J. Roy. Stat. Soc. Ser. B}, 41, 335--340.
\bibAnnoteFile{gittins-79}

\bibitem[{Gittins et~al.(2011)Gittins, Glazebrook, and
  Weber}]{gittins2011multi}
Gittins, John~C., Kevin Glazebrook, and Richard~R. Weber (2011),
  \emph{Multi-armed Bandit Allocation Indices}. John Wiley \& Sons, West
  Sussex, U.K.
\bibAnnoteFile{gittins2011multi}

\bibitem[{Honda and Takemura(2010)}]{honda2010}
Honda, Junya and Akimichi Takemura (2010), \enquote{An asymptotically optimal
  bandit algorithm for bounded support models.} In \emph{COLT}, 67 -- 79,
  Citeseer.
\bibAnnoteFile{honda2010}

\bibitem[{Honda and Takemura(2011)}]{honda2011asymptotically}
Honda, Junya and Akimichi Takemura (2011), \enquote{An asymptotically optimal
  policy for finite support models in the multiarmed bandit problem.}
  \emph{Machine Learning}, 85, 361 -- 391.
\bibAnnoteFile{honda2011asymptotically}

\bibitem[{Jouini et~al.(2009)Jouini, Ernst, Moy, and Palicot}]{jouini2009multi}
Jouini, Wassim, Damien Ernst, Christophe Moy, and Jacques Palicot (2009),
  \enquote{Multi-armed bandit based policies for cognitive radio's decision
  making issues.} In \emph{3rd international conference on Signals, Circuits
  and Systems (SCS)}.
\bibAnnoteFile{jouini2009multi}

\bibitem[{Kaufmann(2015)}]{kaufmann14}
Kaufmann, Emilie (2015), \enquote{Analyse de strat\'egies {B}ay\'esiennes et
  fr\'equentistes pour l'allocation s\'equentielle de ressources.} {\em
  Doctorat}, ParisTech.
\bibAnnoteFile{kaufmann14}

\bibitem[{Lagoudakis and Parr(2003)}]{lagoudakis2003least}
Lagoudakis, Michail~G and Ronald Parr (2003), \enquote{Least-squares policy
  iteration.} \emph{The Journal of Machine Learning Research}, 4, 1107 -- 1149.
\bibAnnoteFile{lagoudakis2003least}

\bibitem[{Lai and Robbins(1985)}]{lai85}
Lai, Tze~Leung and Herbert Robbins (1985), \enquote{Asymptotically efficient
  adaptive allocation rules.} \emph{Advances in {A}pplied {M}athematics}, 6, 4
  -- 2.
\bibAnnoteFile{lai85}

\bibitem[{Li et~al.(2014)Li, Munos, and Szepesv{\'a}ri}]{2014minimax}
Li, Lihong, Remi Munos, and Csaba Szepesv{\'a}ri (2014), \enquote{On minimax
  optimal offline policy evaluation.} arXiv preprint arXiv:1409.3653.
\bibAnnoteFile{2014minimax}

\bibitem[{Littman(2012)}]{littman2012inducing}
Littman, Michael~L (2012), \enquote{Inducing partially observable {M}arkov
  decision processes.} In \emph{ICGI}, 145 -- 148.
\bibAnnoteFile{littman2012inducing}

\bibitem[{Osband and Van~Roy(2014)}]{osband2014near}
Osband, Ian and Benjamin Van~Roy (2014), \enquote{Near-optimal reinforcement
  learning in factored mdps.} In \emph{Advances in Neural Information
  Processing Systems}, 604 -- 612.
\bibAnnoteFile{osband2014near}

\bibitem[{Robbins(1952)}]{Rb52}
Robbins, Herbert (1952), \enquote{Some aspects of the sequential design of
  experiments.} \emph{Bull. Amer. Math. Monthly}, 58, 527--536.
\bibAnnoteFile{Rb52}

\bibitem[{Tekin and Liu(2012)}]{tekin2012approximately}
Tekin, Cem and Mingyan Liu (2012), \enquote{Approximately optimal adaptive
  learning in opportunistic spectrum access.} In \emph{INFOCOM, 2012
  Proceedings IEEE}, 1548 -- 1556, IEEE.
\bibAnnoteFile{tekin2012approximately}

\bibitem[{Tewari and Bartlett(2008)}]{optimistic-mdp}
Tewari, Ambuj and Peter~L Bartlett (2008), \enquote{Optimistic linear
  programming gives logarithmic regret for irreducible mdps.} In \emph{Advances
  in Neural Information Processing Systems}, 1505 -- 1512.
\bibAnnoteFile{optimistic-mdp}

\bibitem[{Weber(1992)}]{weber1992gittins}
Weber, Richard~R (1992), \enquote{On the {G}ittins index for multiarmed
  bandits.} \emph{The Annals of Applied Probability}, 2, 1024 -- 1033.
\bibAnnoteFile{weber1992gittins}

\end{thebibliography}
\ \\  

{\bf Acknowledgement:} 
We are grateful  for support of this project
by  the National Science Foundation, NSF grant CMMI-14-50743.

\appendix
\section{Appendix Proofs}

\begin{proof}[of Theorem \ref{thm:lower-bound}.]
It suffices to demonstrate that for any choice of $\{ f_i \} \subset \mathcal{F}$, for any sub-optimal $i$,
\begin{equation}\label{eqn:lower-bound-sub-optimal}
\liminf_n \frac{ \bE\left[ T^i_\pi(n) \right]  }{ \ln n } \geq \frac{ 1 }{ \inf_{g \in \mathcal{F}} \{ \mathbf{I}(f_i, g) : s(g) > s^* \} }.
\end{equation}

Note, by \Assumption B1, the above infimum exists and is finite. (We note the above is vacuously true if $\inf_{g \in \mathcal{F}} \{ \mathbf{I}(f_i, g) : s(g) > s^* \} = \infty$.) That being so, let $g$ be such that $g \in \mathcal{F}$, $s(g) > s^*$, and $\mathbf{I}(f_i, g) < \infty$. Note, by \Assumption B2, since $s(g) > s^* > s(f_i)$, $\mathbf{I}(f_i, g) > 0$. It will suffice then to show that
\begin{equation}
\liminf_n \frac{ \bE\left[ T^i_\pi(n) \right]  }{ \ln n } \geq \frac{ 1 }{ \mathbf{I}(f_i, g) },
\end{equation}
and take the supremum of the lower bound over feasible $g$. Noting that $\mathbb{E}\left[ T^i_\pi(n) \mathbf{I}(f_i, g) \right]/\ln n \geq \mathbb{P}( T^i_\pi(n) \mathbf{I}(f_i, g) \geq \ln n)$, it would suffice to show that
\begin{equation}
\liminf_n\ \mathbb{P}\left( \frac{ T^i_\pi(n) }{ \ln n } \geq \frac{1}{\mathbf{I}(f_i, g)} \right)  = 1,
\end{equation}
or equivalently that for $0 < \delta < 1$,
\begin{equation}
\limsup_n\ \mathbb{P}\left( \frac{ T^i_\pi(n) }{ \ln n } \leq \frac{1-\delta}{\mathbf{I}(f_i, g)} \right)  = 0.
\end{equation}
Define the following events:
\begin{equation}
A^{\delta}_n = \left\{ T^i_\pi(n) \leq \frac{1-\delta}{ \mathbf{I}(f_i, g)} \ln n \right\},
\end{equation}
\begin{equation}
C^\delta_n = \left\{ \sum_{t = 1}^{ T^i_\pi(n) } \ln \left( \frac{ f_i( X^i_t) }{ g( X^i_t ) } \right) \leq (1 - \delta/2) \ln n \right\}.
\end{equation}
It is additionally convenient to define the sequence of constants $b_n = (1-\delta)/\mathbf{I}(f_i, g) \ln n$ and random variables $S^i_k = \sum_{t = 1}^k \ln\left( f_i( X^i_t)/ g(X^i_t) \right)$. Observe the following bounds.
\begin{equation}
\begin{split}
\mathbb{P}\left( A^{\delta}_n \bar{C}^\delta_n \right) &  \leq  \mathbb{P}\left( \max_{ k \leq \lfloor b_n \rfloor} S^i_k > (1 - \delta/2)\ln n\right) \\
& =  \mathbb{P}\left( \max_{ k \leq \lfloor b_n \rfloor} S^i_k / b_n > (1 - \delta/2)\ln n / b_n \right) \\
& =  \mathbb{P}\left( \max_{ k \leq \lfloor b_n \rfloor} S^i_k / b_n > (1 + \frac{ \delta/2 }{1-\delta } ) \mathbf{I}(f_i, g) \right) \\
& \leq  \mathbb{P}\left( \max_{ k \leq \lfloor b_n \rfloor} S^i_k / b_n > (1 + \frac{ \delta }{2 } ) \mathbf{I}(f_i, g) \right)
\end{split}
\end{equation}
It follows that
\begin{equation}
\limsup_n \mathbb{P}\left( A^\delta_n \bar{C}^\delta_n \right) \leq \limsup_m \mathbb{P}\left( \max_{k \leq m} S^i_k/m \geq \left(1 + \frac{\delta}{2} \mathbf{I}(f_i, g)\right) \right) = 0.
\end{equation}
The last inequality follows, observing that since $0 < \mathbf{I}(f_i, g) < \infty$, we have that $S^i_m / m \to \mathbf{I}(f_i, g)$ almost surely. Since $\limsup_m \max_{k \leq m} S^i_k / m \leq \limsup_m S^i_m / m = \mathbf{I}(f_i, g)$ almost surely, convergence in probability as above is guaranteed.

At this point, recall that $\mathbb{P}$ has been defined by the choice of bandit distributions $\{ f_1, \ldots, f_i, \ldots, f_N \} \subset \mathcal{F}$. Consider an alternative set of distributions, constructed by replacing $f_i$ with $g$: $\{ f_1, \ldots, g, \ldots, f_N \} \subset \mathcal{F}$, and let $\tilde{\mathbb{P}}$ be defined by this alternative set of bandit distributions. The following holds:
\begin{equation}
\begin{split}
\mathbb{P}\left( A^{\delta}_n C^\delta_n \right) & = \mathbb{P}\left( T^i_\pi(n) \leq \frac{1-\delta}{\mathbf{I}(f_i, g)} \ln n, \prod_{t = 1}^{T^i_\pi(n)} f_i(X^i_t) \leq n^{1 - \delta/2} \prod_{t = 1}^{T^i_\pi(n)} g(X^i_t) \right) \\
& \leq \tilde{\mathbb{P}}\left( T^i_\pi(n) \leq \frac{1-\delta}{\mathbf{I}(f_i, g)} \ln n \right) n^{1 - \delta/2}.
\end{split}
\end{equation}
This change of measure argument follows, as $C^\delta_n$ restricts the region of probability space of interest to that where the comparison of densities of bandit $i$ holds. Observing that under this alternative set of bandit densities, bandit $i$ is the unique optimal bandit (since $s(g) > s^*$), and hence $T^i_\pi(n) = n - \bT_\pi(n)$:
\begin{equation}
\begin{split}
\mathbb{P}\left( A^{\delta}_n C^\delta_n \right) & \leq \tilde{\mathbb{P}}\left( n - \frac{1-\delta}{\mathbf{I}(f_i, g)} \ln n  \leq  \bT_\pi(n) \right) n^{1 - \delta/2}.
\end{split}
\end{equation}
For $n$ sufficiently large, so that $n > (1-\delta)/\mathbf{I}(f_i, g) \ln n$, we may apply Markov's inequality to the above (letting $\tilde{\mathbb{E}}$ be expectation under the alternative bandit distribution set):
\begin{equation}
\begin{split}
\mathbb{P}\left( A^{\delta}_n C^\delta_n \right) & \leq  \frac{ \tilde{ \mathbb{E} }\left[  \bT_\pi(n)  \right] }{ n - \frac{1-\delta}{\mathbf{I}(f_i, g)} \ln n } n^{1 - \delta/2} = \frac{ \tilde{ \mathbb{E} }\left[  \bT_\pi(n)  \right] n^{-\delta/2} }{ 1 - \frac{1-\delta}{\mathbf{I}(f_i, g)} \frac{\ln n}{n} }.
\end{split}
\end{equation}
Observing that under the \assumption that $\pi$ is UF, $\tilde{\mathbb{E}}[ \bT_\pi(n) ] = o( n^{\delta/2})$, it follows from the above that $\limsup_n \mathbb{P}\left( A^{\delta}_n C^\delta_n \right) = 0$. Hence,
\begin{equation}
\limsup_n\ \mathbb{P}\left( \frac{ T^i_\pi(n) }{ \ln n } \leq \frac{1-\delta}{\mathbf{I}(f_i, g)} \right)  \leq \limsup_n \mathbb{P}\left( A^{\delta}_n C^\delta_n \right) + \limsup_n \mathbb{P}\left( A^{\delta}_n \bar{C}^\delta_n \right) = 0.
\end{equation}
\qed
\end{proof}

\begin{proof}[of Lem. \ref{lem:bound}]
We recall the definition of $\mathbb{M}_f(\rho)$, and introduce a companion function, $C_f(\delta)$:
\begin{equation}
\begin{split}
\mathbb{M}_f(\rho) & = \inf_{g \in \mathcal{F}} \left\{ \mathbf{I}(f,g) : s(g) > \rho \right\},\\
C_f(\delta) & = \sup_{g \in \mathcal{F}} \left\{ s(g) : \mathbf{I}(f,g) < \delta \right\}.
\end{split}
\end{equation}
Thinking of $\mathbb{M}_f(\rho)$ as the minimal distance (relative to $\mathbf{I}$) from $f$ to a density better than $\rho$, we may consider $C_f(\delta)$ to be the best score achieved within distance $\delta$ of $f$. Note, we have the following relationship: $u_i(n, t) = C_{\hat{f}^i_t}( \ln n / (t - \tilde{d}(t)) )$. Note as well, $\mathbb{M}_f(\rho)$ is an increasing function with $\rho$, and $\mathbb{M}_f( C_f(\delta)) \leq \delta$.

Consider a set of bandit distributions $\{ f_i \} \subset \mathcal{F}$, with $i$ a sub-optimal bandit and $i^*$ an optimal bandit. Let $\epsilon, \delta$ be feasible as in the statement of the Lemma. We define the following functions, for $n \geq n_0 N$:
\begin{equation}
\begin{split}
n^i_1(n, \epsilon, \delta) & = \sum_{t = n_0N}^n \mathbbm{1}\left\{ {\pi^*}(t+1) = i,  u_i(t, T^i_{\pi^*}(t)) \geq s^* - \epsilon, \mathbf{\nu}( \hat{f}^i_{ T^i_{\pi^*}(t) }, f_i ) \leq \delta \right\} \\
n^i_2(n, \epsilon, \delta) & = \sum_{t = n_0N}^n \mathbbm{1}\left\{ {\pi^*}(t+1) = i,  u_i(t, T^i_{\pi^*}(t)) \geq s^* - \epsilon, \mathbf{\nu}( \hat{f}^i_{ T^i_{\pi^*}(t) }, f_i ) > \delta \right\} \\
n^i_3(n, \epsilon) & = \sum_{t = n_0N}^n \mathbbm{1}\left\{ {\pi^*}(t+1) = i,  u_i(t, T^i_{\pi^*}(t)) < s^* - \epsilon \right\}.
\end{split}
\end{equation}
Note the relation that $T^i_{\pi^*}(n+1) = n_0 + n^i_1(n, \epsilon, \delta) + n^i_2(n, \epsilon, \delta) + n^i_3(n, \epsilon)$.

We have the following relations:
\begin{equation}
\begin{split}
\left\{  u_i(t, k) \geq s^* - \epsilon, \mathbf{\nu}( \hat{f}^i_{ k }, f_i ) \leq \delta \right\} & = \left\{  C_{ \hat{f}^i_k }( \ln t / (k - \tilde{d}(k)) ) \geq s^* - \epsilon, \mathbf{\nu}( \hat{f}^i_{ k }, f_i ) \leq \delta \right\} \\
& = \left\{  \mathbb{M}_{ \hat{f}^i_k} ( C_{ \hat{f}^i_k }( \ln t / (k - \tilde{d}(k)) ) ) \geq \mathbb{M}_{\hat{f}^i_k }(s^* - \epsilon), \mathbf{\nu}( \hat{f}^i_{ k }, f_i ) \leq \delta \right\} \\
& \subset \left\{  \ln t / (k - \tilde{d}(k)) \geq \mathbb{M}_{\hat{f}^i_k }(s^* - \epsilon), \mathbf{\nu}( \hat{f}^i_{ k }, f_i ) \leq \delta \right\} \\
& \subset \left\{  \ln t / (k - \tilde{d}(k)) \geq \inf_{g \in \mathcal{F}} \left\{ \mathbb{M}_{g}(s^* - \epsilon) : \mathbf{\nu}( g, f_i ) \leq \delta \right\}  \right\} \\
& = \left\{  \ln t / \inf_{g \in \mathcal{F}} \left\{ \mathbb{M}_{g}(s^* - \epsilon) : \mathbf{\nu}( g, f_i ) \leq \delta \right\} + \tilde{d}(k) \geq k   \right\} \\
\end{split}
\end{equation}
This gives us the following bounds:
\begin{equation}
\begin{split}
& n^i_1(n, \epsilon, \delta) \\
& \leq \sum_{t = n_0N}^n \mathbbm{1}\left\{ {\pi^*}(t+1) = i, \frac{ \ln t }{ \inf_{g \in \mathcal{F}} \left\{ \mathbb{M}_{g}(s^* - \epsilon) : \mathbf{\nu}( g, f_i ) \leq \delta \right\} }+ \tilde{d}( T^i_{\pi^*}(t) ) \geq T^i_{\pi^*}(t) \right\} \\
& \leq \sum_{t = n_0N}^n \mathbbm{1}\left\{ {\pi^*}(t+1) = i,  \frac{ \ln n }{ \inf_{g \in \mathcal{F}} \left\{ \mathbb{M}_{g}(s^* - \epsilon) : \mathbf{\nu}( g, f_i ) \leq \delta \right\} }+ \tilde{d}( T^i_{\pi^*}(t) ) \geq T^i_{\pi^*}(t) \right\} \\
& \leq \sum_{t = 0}^{n-1} \mathbbm{1}\left\{ {\pi^*}(t+1) = i,  \frac{ \ln n }{ \inf_{g \in \mathcal{F}} \left\{ \mathbb{M}_{g}(s^* - \epsilon) : \mathbf{\nu}( g, f_i ) \leq \delta \right\} }+ \tilde{d}( T^i_{\pi^*}(t) ) \geq T^i_{\pi^*}(t) \right\} + 1 \\
& \leq \max \left\{ T : T - \tilde{d}\left(T\right) \leq  \frac{ \ln n }{ \inf_{g \in \mathcal{F}} \left\{ \mathbb{M}_{g}(s^* - \epsilon) : \mathbf{\nu}( g, f_i ) \leq \delta \right\} } \right\} + 1.
\end{split}
\end{equation}
The last bounds in the above hold with the following reasoning: Viewing $T^i_{\pi^*}(t)$ as the sum of $\mathbbm{1}\{ {\pi^*}(t) = i \}$ terms, the added conditioning in the above indicators restrict how many terms of the above sum can be non-zero. Note, this bound holds almost surely, independent of outcomes. Further then, taking $\tilde{d}$ as positive and increasing, for any positive $C$, we have the relation that $\max \{ T : T - \tilde{d}(T) \leq C \} \leq C + O( \tilde{d}(C) )$. Hence, since $\tilde{d}$ is taken to be sub-linear, 
\begin{equation}
n^i_1(n, \epsilon, \delta)  \leq \frac{ \ln n }{ \inf_{g \in \mathcal{F}} \left\{ \mathbb{M}_{g}(s^* - \epsilon) : \mathbf{\nu}( g, f_i ) \leq \delta \right\} } + o( \ln n ).
\end{equation}

To bound the $n^i_2$ term, observe the following:
\begin{equation}
\begin{split}
n^i_2(n, \epsilon, \delta) & \leq \sum_{t = n_0N}^n \mathbbm{1}\left\{ {\pi^*}(t+1) = i, \mathbf{\nu}( \hat{f}^i_{ T^i_{\pi^*}(t) }, f_i ) > \delta \right\} \\
& = \sum_{t = n_0N}^n \sum_{k = n_0}^t \mathbbm{1}\left\{ {\pi^*}(t+1) = i, \mathbf{\nu}( \hat{f}^i_k, f_i ) > \delta, T^i_{\pi^*}(t) = k  \right\} \\
& = \sum_{t = n_0N}^n \sum_{k = n_0}^t \mathbbm{1}\left\{ {\pi^*}(t+1) = i, T^i_{\pi^*}(t) = k \right\} \mathbbm{1}\left\{ \mathbf{\nu}( \hat{f}^i_k, f_i ) > \delta  \right\} \\
& \leq \sum_{k = n_0}^n \mathbbm{1}\left\{ \mathbf{\nu}( \hat{f}^i_k, f_i ) > \delta  \right\} \sum_{t = k}^n \mathbbm{1}\left\{ {\pi^*}(t+1) = i, T^i_{\pi^*}(t) = k \right\} \\
& \leq \sum_{k = n_0}^n \mathbbm{1}\left\{ \mathbf{\nu}( \hat{f}^i_k, f_i ) > \delta  \right\}.
\end{split}
\end{equation}
To bound the $n^i_3$ term, note that by the structure of the policy, if ${\pi^*}(t+1) = i$, $u_i(t, T^i_{\pi^*}(t)) = \max_j u_j(t, T^j_{\pi^*}(t))$. Hence, if $i^*$ is an optimal bandit, ${\pi^*}(t+1) = i$, and $u_i(t, T^i_{\pi^*}(t)) < s^* - \epsilon$, it must also be that $u_{i^*}(t, T^{i^*}_{\pi^*}(t)) < s^* - \epsilon$. Hence we have the following bound:
\begin{equation}
\begin{split}
n^i_3(n, \epsilon) & \leq \sum_{t = n_0N}^n \mathbbm{1}\left\{ {\pi^*}(t+1) = i,  u_{i^*}(t, T^{i^*}_{\pi^*}(t)) < s^* - \epsilon \right\} \\
& \leq \sum_{t = n_0N}^n \mathbbm{1}\left\{ u_{i^*}(t, T^{i^*}_{\pi^*}(t)) < s^* - \epsilon \right\} \\
& \leq \sum_{t = n_0N}^n \mathbbm{1}\left\{ u_{i^*}(t, k) < s^* - \epsilon \text{ for some $k = n_0, \ldots, t$ } \right\} \\
& \leq \sum_{t = n_0N}^n \sum_{k = n_0}^t \mathbbm{1}\left\{ u_{i^*}(t, k) < s^* - \epsilon \right\}.
\end{split}
\end{equation}
Combining each of the above bounds, and observing that $T^i_{\pi^*}(n) \leq T^i_{\pi^*}(n+1)$, we have for $n \geq n_0N$:
\begin{equation}
\begin{split}
T^i_{\pi^*}(n) \leq  &\  \frac{ \ln n }{ \inf_{g \in \mathcal{F}} \left\{ \mathbb{M}_{g}(s^* - \epsilon) : \mathbf{I}( g, f_i ) \leq \delta \right\} } + o( \ln n ) \\
& +  \sum_{k = n_0}^n \mathbbm{1}\left\{ \mathbf{I}( \hat{f}^i_k, f_i ) > \delta  \right\} \\
& + \sum_{t = n_0N}^n \sum_{k = n_0}^t \mathbbm{1}\left\{ u_{i^*}(t, k) < s^* - \epsilon \right\}.
\end{split}
\end{equation}
Taking expectations completes the proof.\qed
\end{proof}

\begin{proposition}\label{prop:sum-bound}
For $\Delta > 0, \tilde{d}(k) = o(k), t > 1$,
\begin{equation}
\sum_{k = 1}^{\infty} t^{-\Delta/\left(k-\tilde{d}(k)\right)} e^{-\Omega(k)} \leq O( 1 / \ln t ).
\end{equation}
\end{proposition}
\begin{proof}[Proof of Proposition \ref{prop:sum-bound}]
Let $1 > p > 0$. We have
\begin{equation}
\begin{split}
 \sum_{k = 1}^\infty t^{-\Delta/\left(k-\tilde{d}(k)\right)} e^{-\Omega(k)}  & = \sum_{k = 1}^{\left \lfloor{ \ln(t)^p }\right \rfloor }  t^{-\Delta/\left(k-\tilde{d}(k)\right)} e^{-\Omega(k)} + \sum_{k = \left \lceil{ \ln(t)^p }\right \rceil}^\infty  t^{-\Delta/\left(k-\tilde{d}(k)\right)} e^{-\Omega(k)} \\
& \leq \sum_{k = 1}^{\left \lfloor{ \ln(t)^p }\right \rfloor }t^{-\Delta/\left(k-\tilde{d}(k)\right)} + \sum_{k = \left \lceil{ \ln(t)^p }\right \rceil}^\infty e^{-\Omega(k)} \\
& =  \ln(t)^p  e^{-\Omega( \ln(t)^{1-p} ) } +e^{-\Omega(  \ln(t)^p )} .
\end{split}
\end{equation}
Here we may make use of the following bounds, that for $x \geq 0$,
\begin{equation}
\begin{split}
x^p e^{-\Omega(x^{1-p})} & \leq O( 1/x ) \\
e^{-\Omega(x^p)} & \leq O( 1/x ).
\end{split}
\end{equation}
Applying these to the above,
\begin{equation}
\sum_{k = 1}^\infty t^{-\Delta/( k - \tilde{d}(k))} e^{-\Omega(k)} \leq O( 1/\ln(t) ).
\end{equation}\qed
\end{proof}

\begin{proof}[of Lem. \ref{lem:estimators}]
To see the distribution of $\hat{\alpha}_n$, consider the event that $X_1 = \min_t X_t$. This can be generated in the following way, by first generating $X_1$ according to Pareto$(\alpha, \beta)$, then for each $j \neq 1$, generating each $X_j$ independently as Pareto$(\alpha, \beta)$ conditioned on $X_j \geq X_1$, in which case $X_j \sim \text{Pareto}(\alpha, X_1)$, by the self-similarity of the Pareto distribution. Using the standard fact that if $X \sim \text{Pareto}(\alpha, \beta)$, then $\ln(X/\beta) \sim \text{Exp}(\alpha)$, we have that
\begin{equation}
\sum_{t = 1}^n \ln\left( \frac{X_t}{X_1} \right)
\end{equation}
is distributed as the sum of $n-1$ many i.i.d. exponential random variables with parameter $\alpha$, or Gamma$(n-1, \alpha)$. Note, this holds independent of the value of $X_1$. The same argument holds, taking any of the $X_t$ as the minimum. Hence, independent of which $X_t$ is the minimum, and independent of the value of that minimum (i.e., independent of $\hat{\beta}_n$, the above sum is distributed like Gamma$(n-1,\alpha) \sim $ Gamma$(n-1,1)/\alpha$. This gives the above representation of $\hat{\alpha}_n$ and demonstrates the independence of $\hat{\alpha}_n$ and $\hat{\beta}_n$.

To see the distribution of $\hat{\beta}_n$, note that $\hat{\beta}_n \geq \beta,$ and for $x \geq 1$,
\begin{equation}
\bP( \hat{\beta}_n /\beta > x ) = \bP( \hat{\beta}_n > \beta x ) = \prod_{t = 1}^n \bP( X_t > \beta x ) = \left( \frac{\beta}{\beta x} \right)^{n \alpha } = \left( \frac{1}{x} \right)^{n \alpha },
\end{equation}
which shows that $\hat{\beta}_n/\beta \sim \text{Pareto}(\alpha n, 1)$.\qed
\end{proof}

\begin{proof}[of Lem. \ref{lem:gamma-bound}]
Let $Y_1, \ldots, Y_t$ be i.i.d. Exp(1) random variables, and let $G = Y_1 + \ldots + Y_t$. For $0 < \gamma^{-} < 1 < \gamma^{+} < \infty$,
\begin{equation}
\begin{split}
\mathbb{P}\left( G < \gamma^{-}t \right) & = \mathbb{P}\left( e^{-\left(\frac{1}{\gamma^{-}} - 1\right) G } > e^{- \left(\frac{1}{\gamma^{-}} - 1\right) \gamma^{-} t } \right) \\
& = \mathbb{P}\left( e^{-(\frac{1}{\gamma^{-}} - 1) G } > e^{- (1 -  \gamma^{-}) t } \right) \\
& \leq \frac{ \mathbb{E}\left[ e^{-(\frac{1}{\gamma^{-}} - 1) G } \right] }{ e^{- (1 -  \gamma^{-}) t } }  =  \frac{ \prod_{s = 1}^t \mathbb{E}\left[ e^{-(\frac{1}{\gamma^{-}} - 1) Y_s } \right] }{ e^{- (1 -  \gamma^{-}) t } } = \frac{ (\gamma^{-})^t}{ e^{- (1 -  \gamma^{-}) t } } = \left( \gamma^{-} e^{1 - \gamma^{-}} \right)^t.
\end{split}
\end{equation}
The result for $\mathbb{P}\left( G > \gamma^{+}t \right)$ follows similarly.
\qed
\end{proof}


\begin{proof}[of Lemma \ref{lem:uniform}.]
Let $X_1, \ldots, X_t$ be i.i.d. $\text{Uniform}[0,1]$ random variables. Note that we may then take $\hat{a}_t = a + (b-a) \min_n X_n$, $\hat{b}_t = a + (b-a) \max_n X_n$. Hence,
\begin{equation}
\mathbb{P}\left( \frac{ \hat{b}_t - \hat{a}_t }{ b - a } < \lambda \right) = \mathbb{P}\left( \max_n X_n - \min_n X_n < \lambda \right)
\end{equation}
Let $M = \max_n X_n$ and $m = \min_n X_n$. Note that, conditioned on $m$, $M-m$ is distributed like the maximum of $t-1$ many $\text{Uniform}[0,1-m]$ random variables. Let $Y_1, \ldots, Y_{t-1}$ be i.i.d. $\text{Uniform}[0,1]$ random variables, so we may take $M-m = (1-m) \max_s Y_s$.
\begin{equation}
\begin{split}
\mathbb{P}\left( M - m < \lambda | m \right) & = \mathbb{P}\left( (1-m) \max_s Y_s < \lambda | m \right) \\
& = \mathbbm{1}\{ 1-m \leq \lambda \} + \frac{ \lambda^{t-1} }{(1-m)^{t-1}}\mathbbm{1}\{ 1-m > \lambda \}
\end{split}
\end{equation}
Note that $m$ is distributed with a density of $t(1 - x)^{t-1}$ for $x \in [0,1]$. From the above then
\begin{equation}
\begin{split}
\mathbb{P}\left( \frac{ \hat{b}_t - \hat{a}_t }{ b - a } < \lambda \right) & = \mathbb{P}\left( M - m < \lambda \right) \\
& = \mathbb{E}\left[ \mathbb{P}\left( M - m < \lambda | m \right) \right] \\
& = \mathbb{P}\left(1-\lambda \leq m \right) + \mathbb{E}\left[ \frac{ \lambda^{t-1} }{(1-m)^{t-1}}\mathbbm{1}\{ 1- \lambda > m \} \right] \\
& = \lambda^{t} +  t (1-\lambda) \lambda^{t-1}.
\end{split}
\end{equation}
The result follows immediately.\qed
\end{proof}

\begin{proof}[of Lemma \ref{lem:normal-bound}.]
For the normal bound, let $\Phi$ represent the standard normal c.d.f.. It suffices then to demonstrate that for $z \geq 0$, $1 - \Phi(z) \leq e^{-z^2/2}/2$. However, it is easy to show that $2e^{z^2/2}\left( 1 - \Phi(z) \right)$ is a positive, monotonically decreasing function of $z$ over this range, with a maximum of $1$ at $z = 0$.

For the $\chi^2_t$ bounds, let $0 < u^{-} < 1 < u^{+}$, and let $Z_1, \ldots, Z_t$ be i.i.d. standard normal random variables. Let $U_t = \sum_{i = 1}^t Z_i^2$. Observe that
\begin{equation}
\begin{split}
\mathbb{P}\left( U_t > u^+ t \right) & = \mathbb{P}\left( e^{ \left( \frac{1}{2} - \frac{1}{2 u^{+}}\right) U_t } > e^{ \left( \frac{1}{2} - \frac{1}{2 u^{+}}\right) u^{+} t } \right) \\
& =  \mathbb{P}\left( e^{ \left( \frac{1}{2} - \frac{1}{2 u^{+}}\right) U_t } > e^{ \left( u^{+} - 1\right)t/2 } \right) \\
& \leq \mathbb{E}\left[ e^{ \left( \frac{1}{2} - \frac{1}{2 u^{+}}\right) U_t } \right] e^{- \left( u^{+} - 1\right)t/2 } \\
& = \mathbb{E}\left[ e^{ \left( \frac{1}{2} - \frac{1}{2 u^{+}}\right) Z^2 } \right]^t e^{- \left( u^{+} - 1\right)t/2 } \\
& =  \left(\sqrt{ u^{+} }\right)^t e^{- \left( u^{+} - 1\right)t/2 }.
\end{split}
\end{equation}
The result follows immediately as a rearrangement of the above. The result for $\mathbb{P}\left( U_t < u^{-}t \right)$ follows similarly.\qed
\end{proof}

\end{document}